\documentclass{article}
\usepackage[preprint]{neurips_2024}
\usepackage{microtype}
\usepackage{graphicx} 
\usepackage[disable]{todonotes}
\usepackage{subcaption} 
\usepackage{tikz} 
\usetikzlibrary{positioning, calc, fit}
\usepackage{url}

\usepackage{hyperref}

\usepackage{amsthm} 
\usepackage{amsfonts} 
\usepackage{amssymb} 
\usepackage{amsmath} 
\usepackage{xspace} 
\usepackage{mathtools} 
\usepackage{algorithm} 
\usepackage{algorithmic}
\usepackage{booktabs}
\usepackage{enumitem}

\theoremstyle{plain}
\newtheorem{theorem}{Theorem}[section]
\newtheorem*{namedtheorem}{\customthmname}
\newcommand{\customthmname}{}
\newenvironment{customthm}[1]{\renewcommand{\customthmname}{Theorem #1}\namedtheorem}{\endnamedtheorem} 
\newenvironment{theorem*}
 {\expandafter\def\expandafter\thetheorem\expandafter{\thetheorem$^\dagger$}\theorem}
 {\endtheorem} 
\newtheorem{corollary}[theorem]{Corollary}
\newtheorem{lemma}[theorem]{Lemma}

\theoremstyle{definition}
\newtheorem{definition}[theorem]{Definition}

\theoremstyle{remark}
\newtheorem{remark}[theorem]{Remark}

\newtheorem*{namedremark}{\customremname}
\newcommand{\customremname}{}
\newenvironment{customrem}[1]{\renewcommand{\customremname}{Remark #1}\namedremark}{\endnamedremark} 
\author{Author1}

\newcommand{\N}{\ensuremath{\mathbb{N}}\xspace}
\newcommand{\R}{\ensuremath{\mathbb{R}}\xspace}
\newcommand{\bO}{\mathcal{O}}

\newcommand{\NN}{\ensuremath{\mathcal{N}}\xspace}
\newcommand{\C}{\ensuremath{\mathcal{C}}\xspace}
\newcommand{\CiN}{\ensuremath{\mathcal{C}^{(i)}_{\NN}}\xspace}
\newcommand{\Graph}{\ensuremath{\mathfrak{Graph}}\xspace}
\newcommand{\lgraph}{\ensuremath{\mathfrak{G}}\xspace} 
\newcommand{\CGNN}{\text{C-GNN}\xspace}
\newcommand{\actfct}{\ensuremath{\sigma}\xspace}
\newcommand{\faco}{\ensuremath{\mathrm{FAC}^0_{\R^k}}\xspace}
\newcommand{\facoA}{\ensuremath{\mathrm{FAC}^0_{\R^k}[\mathcal{A}]}\xspace}
\newcommand{\facoR}{\ensuremath{\mathrm{FAC}^0_{\R}}\xspace}
\newcommand{\facoAR}{\ensuremath{\mathrm{FAC}^0_{\R}[\mathcal{A}]}\xspace}
\newcommand{\proj}[2][]{\ensuremath{\textit{proj}_{#1}{#2}\xspace}}
\newcommand{\size}{\textit{size}}
\newcommand{\depth}{\textit{depth}}
\newcommand{\SD}[1]{\mathrm{FSIZE}\text{-}\mathrm{DEPTH}_{#1}}
\newcommand{\SDk}{\SD{\R^k}}

\newcommand{\TCzero}{\mathrm{TC}^0}
\newcommand{\lgc}[1]{\ensuremath{\lgraph_{{#1}, \overline{x}}}} 

\newcommand{\ol}[1]{\overline{#1}}
\newcommand{\enc}[1]{\ensuremath{\langle #1 \rangle}}

\newcommand{\lms}{\ensuremath{\left\{\!\!\left\{}}
\newcommand{\rms}{\ensuremath{\right\}\!\!\right\}}}

\NewDocumentCommand{\ms}{om}{%
  \sbox0{\mathsurround=0pt$#1\{$}%
  \sbox2{\{}%
  \ifdim\ht0=\ht2
    \{\kern-.625\wd2 \{#2\}\kern-.625\wd2 \}%
  \else
    \mathopen{#1\{\kern-.7\wd0 #1\{}
    #2
    \mathclose{#1\}\kern-.7\wd0 #1\}}
  \fi
}

\newcommand{\jonnicolor}{red!20}
\newcommand{\lauracolor}{green!20}

\newcommand{\heribertcolor}{teal!20}
 
\newcommand{\jonniS}[1]{\todo[size=\tiny, color=\jonnicolor]{Jonni: #1}}
\newcommand{\lauraS}[1]{\todo[size=\tiny, color=\lauracolor]{Laura: #1}}

\newcommand{\heribertS}[1]{\todo[size=\tiny, color=\heribertcolor]{Heribert: #1}}

\DeclareUnicodeCharacter{211D}{$\mathbb{R}$}

\title{Graph Neural Networks and Arithmetic Circuits}

\author{%
  Timon Barlag\\
  Institute for Theoretical Computer Science\\
  Leibniz University Hanover\\
  Hanover, Germany \\
  \texttt{barlag@thi.uni-hannover.de} \\
   \And
  Vivian Holzapfel\\
  Institute for Theoretical Computer Science\\
  Leibniz University Hanover\\
  Hanover, Germany \\
  \texttt{holzapfel@thi.uni-hannover.de} \\
   \And
  Laura Strieker\\
  Institute for Theoretical Computer Science\\
  Leibniz University Hanover\\
  Hanover, Germany \\
  \texttt{strieker@thi.uni-hannover.de} \\
   \And
  Jonni Virtema\\
  School of Computer Science\\
  University of Sheffield\\
  Sheffield, United Kingdom\\
  \texttt{j.t.virtema@sheffield.ac.uk} \\
   \And
  Heribert Vollmer\\
  Institute for Theoretical Computer Science\\
  Leibniz University Hanover\\
  Hanover, Germany \\
  \texttt{vollmer@thi.uni-hannover.de} \\
}

\begin{document}
\maketitle






\begin{abstract}
We characterize the computational power of neural networks that follow the graph neural network (GNN) architecture, not restricted to aggregate-combine GNNs or other particular types.
We establish an exact correspondence between the expressivity of GNNs using diverse activation functions and arithmetic circuits over real numbers.
In our results the activation function of the network becomes a gate type in the circuit. 
Our result holds for families of constant depth circuits and networks, both uniformly and non-uniformly, for all common activation functions.

\end{abstract}


\section{Introduction}


\setcounter{footnote}{2}
Neural networks have recently received growing attention from a theoretical point of view in a number of papers studying their computational properties.
Relevant to this paper are examinations of the computational power of neural networks after training, i.e., the training process is not taken into account but instead the computational power of an optimally trained network is studied. 
Starting already in the nineties, the expressive power of feed-forward neural networks (FNNs) has been related to Boolean threshold circuits, see, e.g., \citep{DBLP:conf/focs/MaassSS91,DBLP:journals/jcss/SiegelmannS95,DBLP:journals/siamcomp/Maass97,DBLP:journals/eccc/HeP22}. 
Most importantly, \citet{DBLP:journals/siamcomp/Maass97} showed that when restricting networks to Boolean inputs, a language can be decided by a family of (in a certain sense) ``polynomial-size'' FNNs if and only if it belongs to the class $\TCzero$, i.e., the class of all languages decidable by families of Boolean circuits of constant depth and polynomial size using negation gates and unbounded fan-in AND, OR, and threshold gates. 

In the last five years attention has shifted to the study of \textit{graph neural networks (GNNs)} which is a model for machine learning tasks on graph-structured inputs.
Their expressive power can be closely related to the Weisfeiler-Leman algorithm \citep{DBLP:conf/iclr/XuHLJ19,DBLP:conf/aaai/0001RFHLRG19,DBLP:conf/ijcai/0001FK21}. 
Originally developed to solve the graph isomorphism problem, the WL algorithm is connected to fragments of first-order logic; hence the logical expressiveness of GNNs was subsequently studied. 
\citet{Barcelo_2020} consider so-called logical classifiers---these are unary formulas of first-order (FO) predicate logic that classify a node in a given graph according to whether the formula holds for this node. 
Barcel\'o et al.{} proved that all classifiers definable in a certain fragment GC (guarded first-order logic with counting quantifiers)
are computable by GNNs.
The converse direction is a bit problematic;
it is only known to hold for unary queries definable in first-order logic.
This result has been broadened recently by \citep{DBLP:conf/icalp/BenediktLMT24} using logics with Presburger quantifiers. 
Barcel\'o et al.{} then consider an extension of GNNs by so-called global readout and show that these GNNs can compute all queries from the logic FOC$_2$ (first-order restricted to two variables, but extended with counting quantifiers, a superclass of GC), but the converse is open.

\textbf{Our motivation.}
Neural networks in general and GNNs specifically have mostly been studied through the lens of Boolean functions. 
This does not necessarily reflect on the usage of machine learning models in real-world applications. 
We want to shift this focus to real-valued computations and hence in this paper, we make real numbers first-class citizens.
When networks are digitally simulated in practical applications, of course only rational numbers are used as inputs, and the real numbers that appear during their computation (e.g., via the sigmoid function) are approximated by rationals. However, we are interested in principal statements about the computational power of neural networks---hence we study networks as devices operating with real numbers.
We do not solely consider networks that are restricted to Boolean inputs or Boolean outputs (Boolean queries, logical classifiers), as done in all the papers cited above.
Instead, we consider GNNs as a model to compute functions from (vectors of) real numbers to (vectors of) real numbers, or from labeled graphs (i.e., undirected graphs whose nodes are annotated with vectors of real numbers) to labeled graphs.

Instead of turning to logics, we focus on a computation model that has been used to capture the expressivity of neural networks in the past:
we use circuits, but since we turn away from the Boolean computation model, we use \textit{arithmetic circuits} over the reals, that is, circuits that take real numbers as inputs and have nodes that compute real functions such as addition, multiplication, projection, or a constant function. 
Why do we do that?
It will turn out that in this way we obtain a close correspondence between GNNs and arithmetic circuits. 
Arithmetic circuits can in a second step be simulated by a Boolean computation model, but this does not say anything about the computational power of the networks.
Going directly from neural networks to Boolean circuits mixes up two different issues and obscures statements about the power of GNNs.
By separating the two aspects, we do not only obtain an upper bound, or correspondence with respect to discrete classification tasks, but an equivalence between GNNs and constant-depth arithmetic circuits over the real functions they compute. In this way, we shed more light on the computational model behind GNNs. 
Our results show very explicitly what the computational abilities of GNNs are and what elementary operations they can perform.

We also want to make a general statement about scaling and complexity.
A common narrative nowadays seems to be that by making neural networks larger and larger, they can solve more complex reasoning tasks. 
However, theoretical limitations as proven in this paper show that scaling is not all we need. Our results show that to improve the expressivity of GNNs, more computationally complex aggregation and combine functions are needed.
More precisely, these functions need to come outside of the class \facoR, which contains functions that can be computed by arithmetic circuits that are bounded in size and depth, see Definition \ref{def:fac0} for a precise definition. 
Another approach would be to change the fixed depth framework of GNNs to one that includes recursion, for example.

Following a broad line of research outlined by \cite{MerrillYouTube22},
in this paper, we aim at a general statement about the computational power of neural networks following the GNN architecture, i.e., networks operating in layers where connections within layers and hence their communication capabilities correspond to the given input graph.
However, contrary to most predecessor papers \citep{DBLP:conf/aaai/0001RFHLRG19,Barcelo_2020}, we do not only consider so-called AC-GNNs, i.e., GNNs where the computation in each node of a layer of the network consists first of an aggregation step (collecting information from the adjacent nodes of the graph---very often simply summation), followed by a combine step (combining the aggregated information with the current information of the node---mostly a linear function followed by a unary non-linear activation function such as sigmoid or ReLU). 
AC-GNNs like these are just one example of a GNN architecture.
We take a more general approach. 
The GNN framework defines a continuous graph-based ``message passing'' \citep{gilmer2017neural, DBLP:conf/aaai/0001RFHLRG19}.
In our networks we keep the layout and message passing mechanisms of GNNs but equip the nodes in each layer with constant-depth arithmetic circuits (taken from a previously fixed basis of circuits), not limited to the usual aggregate-combine (AC) operations. 
We call these networks C-GNNs, circuit graph neural networks.
AC-GNNs whose aggregation function is computable by a constant-depth circuit, can be simulated by C-GNNs (using the same activation function).
The computational upper bounds we obtain thus point out general computational limitations for networks following the GNN message passing framework.

\textbf{Main Results.}
The main contribution of our paper is an exact correspondence between C-GNNs and arithmetic circuits.
A function (from labeled graphs to labeled graphs) is computable by a C-GNN with a constant number of layers if and only if it is computable by a constant-depth arithmetic circuit over the real numbers. 
The activation function of the neural network will be a gate type in the arithmetic circuit. 
Thus, the (set of) actual activation function(s) becomes a parameter in a general statement equating the computational power of GNNs and arithmetic circuits. 

A number of remarks are in order.
\begin{itemize}[noitemsep,topsep=0pt]\def\labelitemi{--}
    \item Our result holds for all commonly used activation functions.
    \item Our result is uniform. If we start with a family of GNNs whose individual networks can be generated via some algorithm, the resulting family of arithmetic circuits will be uniform in the same way, and vice versa.
    \item Our result gives general limits for neural networks following the GNN message passing framework, including but not restricted to the currently widely used AC-GNNs. This means that in order to use a GNN to compute a function not computable by constant-depth arithmetic circuits, scaling GNNs up or adding computational power in the individual nodes will not help, but neural networks of a completely different architecture are needed.
\end{itemize}

This paper follows a modular approach. We stick to the GNN message passing framework, but allow computations of arbitrary power in the individual nodes of the graph. Our approach is adaptable and opens possibilities for many extensions; we point out a few in our conclusion section. 

We would like to stress that theoretical insights like those presented in this paper have the potential to guide practical development. We prove principal limits of networks following the GNN framework,  extending the current AC-GNNs. Our results enables the use of theoretical results regarding arithmetic circuit complexity for arguing about the complexity of GNNs. For example findings for a complexity hierarchy within the real valued \facoR will lead to corresponding statements about the power of GNNs. They may also be used to design GNN architectures with provable expressivity bounds.\looseness=-1

\textbf{Organization.}
In the next section, we introduce the relevant notions about GNNs and arithmetic circuits. 
Section~\ref{main} contains our results. 
In Subsection~\ref{sec:model_of_comp}, we introduce Circuit-GNNs as a generalization of AC-GNNs. 
In Subsections~\ref{sec:sim_cgnn_with_circ} and \ref{sec:sim_circ_with_cgnn}, we show how to simulate C-GNNs by arithmetic circuits, and vice versa.
In Section~\ref{sec:conclusion}, we point out a number of questions for further research.
Due to lack of space, some proofs marked with $\dagger$ are deferred to the appendix.
\jonniS{Do we want to add the appendix to the paper or leave it out? Reference would then be to arxiv.}
\heribertS{Do we plan to publish the paper in a journal? If yes, then we should definitely leave out the appendix here. But in any case, I think I am in favor of leaving it out for now.}

\textbf{Related Work.}
The first model of GNNs was introduced by \cite{DBLP:journals/tnn/ScarselliGTHM09}. 
Since then, their expressive power has been studied in numerous works and from various perspectives as discussed above.
As mentioned already, \citet{Barcelo_2020} established a close connection between classifiers definable in guarded FOC$_2$ and computable by AC-GNNs. It is worth to point out that a connection very close to the aforementioned one can be obtained by connecting the results of \citet{DBLP:conf/nips/SatoYK19}, that relate graph neural networks with more general models of distributed computation, to the results of \citet{DBLP:journals/dc/HellaJKLLLSV15}, that logically characterize expressivity of those models of distributed computation with graded modal logics (known to be equivalent with guarded FOC$_2$).
\cite{DBLP:conf/lics/Grohe23} very recently extended these characterization results for GNNs, and furthermore connected GNNs to Boolean circuits. 
He considers an FO-fragment GFO+C (guarded fragment of FO plus a more general form of counting than used by Barcel\'o et al.) and proves equivalence of this logic to GNNs for unary queries, that is for Boolean functions (functions with Boolean output). 
However, his overall result holds only in the non-uniform setting
(that is, the sequence of GNNs for graphs of different sizes is non-uniform in the sense that there is not necessarily an algorithm 
that, given a graph size, computes that network from the sequence responsible for the given size), and the activation functions have to be ``rpl approximable'' (i.e., rational piecewise linear approximable, i.\,e., the real functions must in a certain way be a approximable by Boolean circuits or logic). 
This includes most of the commonly used activation functions. 
Since queries expressible in this logic are known to be $\TCzero$ computable, he obtains a $\TCzero$ upper bound for the computational power of GNNs. 
For the converse connection, an extension of GNNs is necessary, and it is shown that a unary query computable by a $\TCzero$ circuit is computable by a GNN with so-called random initialization.
Another line of work that focuses on real-valued computations includes \citep{DBLP:conf/foiks/GeertsSB22} where the logic MPLang operating directly on real numbers was introduced to readily express GNN computations.
The main goal of that paper was to reason about GNNs using a logic, extended by real-number computations. 
However, the converse connection, from MPLang to GNNs, is highly unclear. 
Other work addressing the precise characterizations of expressivity issues of GNNs outside of the restriction to the Boolean setting with GNNs using formal logic include \citep{cucala2023correspondence, pfluger2024recurrent}. \lauraS{or cite Benedikt here?}
To the best of our knowledge, there are no other works considering real-valued computations and circuits or arithmetic circuits.
Another approach that leans more into comparing the expressivity when using different functions for aggregation in the GNN can be found in \cite{DBLP:journals/corr/abs-2302-11603}, where they compared the expressivity of GNNs using sum as their aggregation with mean and max GNNs.

\section{Preliminaries}
\label{sec:prelim}

Throughout this paper, the graphs we consider have an ordered set of vertices and are undirected unless otherwise specified. 
We use overlined letters to denote tuples, write $\lvert \overline{x} \rvert$ to denote the length of $\ol{x}$ (i.e., the number of elements of $\ol{x}$), and use $[n]$ to denote the first $n$ nonzero natural numbers $\{1, \dots, n\}$. 
For a $k \in \N$ we denote all possible $n$-tuples for every $n$ of $\R^k$ with $\left(\R^k\right)^*$.
The notation $\lms \rms$ is used for multisets. A restriction of a function $f$ to set $S$ is written as $f\!\! \restriction_S$.



\begin{definition}
    Let $k \in \N$, and let $G=(V,E)$ be a graph with an ordered set of vertices $V$ and a set of undirected edges $E$ on $V$. 
    Let $g_V \colon V \to \R^k$ be a function which labels the vertices with attribute vectors.
    We then call $\lgraph = (V, E, g_V)$ a \emph{labeled graph of dimension $k$} and denote by $\Graph$ the set of all labeled graphs and by $\Graph_k$ the set of all labeled graphs of dimension $k$. 
\end{definition}
For a node $v \in V$, the \emph{neighborhood of $v$} is defined as
      $N_{\lgraph}(v) \coloneqq \{w \in V \mid \{v, w\} \in E \}$.

In order to compare standard graph neural networks and the circuit graph neural networks which we will introduce later on, we fix the notion of a graph neural network first. 
Generally, GNNs can classify either individual nodes or whole graphs. 
For the purpose of this paper, we will only introduce node classification as graph classification works analogously.
We focus on aggregate combine graph neural networks that aggregate the information of every neighbor of a node and then combine this information with the information of the node itself.

An \emph{aggregation function} (of dimension $k$) is a permutation-invariant function $\text{AGG}: \R^k \times \dots \times \R^k \to \R^k$, a \emph{combine function} (of dimension $k$) is a function $\text{COM}: \R^k \times \R^k \to \R^k$ and a \emph{classification function} (of dimension $k$) is a function $\text{CLS}: \R^k \to \{0,1\}$ that classifies a real vector as either true or false.
Finally, an \emph{activation function} (of dimension $k$) is a componentwise 
function $\actfct\colon \R^k \to \R^k$.
\begin{definition}[AC-GNN, cf. \citep{Barcelo_2020}]
    An $L$ layer  \emph{aggregate combine graph neural network} (AC-GNN) is  a tuple $\mathcal{D}=(\{ \text{AGG}^{(i)}\}_{i=1}^L, \{ \text{COM}^{(i)} \}_{i=1}^L, \{\actfct^{(i)}\}_{i=1}^L, \allowbreak\text{CLS})$, where $\{ \text{AGG}^{(i)}\}_{i=1}^L$ and $\{ \text{COM}^{(i)} \}_{i=1}^L$ are sequences of aggregation and combine functions, $\{\actfct^{(i)}\}_{i=1}^L$ is a sequence of activation functions and CLS is a classification function.

    Given a labeled graph $\lgraph=(V,E, g_V)$,
    the AC-GNN model computes vectors $\ol{x}_v^{(i)}$ for every $v \in V$ in every layer $1\leq i \leq  L$ as follows:
    $\ol{x}_v^{(0)}=g_V(v)$ is the initial feature vector of $v$, and for $i > 1$
    \begin{align*}   
     &\ol{x}_v^{(i)} = \actfct^{(i)} \left( \text{COM}^{(i)}\left( \ol{x}_v^{(i-1)}, \ol{y} \right)\right)\text{, where } \ol{y} = \text{AGG}^{(i)} \left(\ms[\big]{ \ol{x}_u^{(i-1)} \mid u \in N_\lgraph(v) } \right). 
     \end{align*}
    The classification function CLS: $\R^k \to \{0,1\}$ is applied to the resulting feature vectors $\ol{x}_v^{(L)}$. 
        \label{def:ac_gnn}
\end{definition} 

In this paper we focus on the real-valued computation part of GNNs and discard the classification function. 
We consider the feature vectors $\ol{x}_v^{(L)}$ after the computation of layer $L$ as our output.
While we could also integrate CLS into our model, for our concerns this is not needed.


Next we define arithmetic circuits as a model of computation for computing real functions. 
Since we will put them in the context of graph neural networks, which inherently operate on vectors of real numbers rather than individual reals, we will accordingly define arithmetic circuits relative to $\R^k$ rather than relative to $\R$, as is done more commonly (cf. e.g. \citep{ComplexityAndRealComputation}).

\begin{definition}
    Let $k,n,m \in \mathbb{N}$.
    An \emph{$\R^k$-arithmetic circuit} with $n$ inputs and $m$ outputs is a simple directed acyclic graph of labeled nodes, also called \emph{gates}, such that
    \begin{itemize}[noitemsep,topsep=0pt]
        \item there are exactly $n$ gates labeled \emph{input}, which each have indegree $0$,
        \item there are exactly $m$ gates labeled \emph{output}, which have indegree $1$ and outdegree $0$,
        \item there are gates labeled \emph{constant}, which have indegree $0$ and are labeled with a tuple $c \in \R^k$,
        \item there are gates labeled \emph{projection$_{i, j}$} for $1 \leq i,j \leq k$, which have in- and outdegree $1$,
        \item there are gates labeled \emph{addition} and \emph{multiplication}.
    \end{itemize}
    Additionally, both the input and the output gates are ordered.
    
    An $\R^k$-arithmetic circuit $C$ with $n$ inputs and $m$ outputs computes the function $f_C \colon (\R^k)^n \to (\R^k)^m$ as follows:
    Initially, the input to the circuit is placed in the input gates.
    Afterwards, in each step, each arithmetic gate whose predecessor gates all have a value, computes the respective function it is labeled with, using the values of its predecessors as inputs.
    By addition and multiplication we refer to the respective componentwise operations and
    the projection computes the function $\proj[i,j]{} \colon \R^k \rightarrow \R^k,\: (x_1, \dots, x_i, \dots x_k) \mapsto (0, \dots, 0, \underset{\text{position }j}{x_i}, 0, \dots, 0)$.
    Analogously, each output gate takes the value of its predecessor, once its predecessor has one. 
    The output of $f_C$ is then the tuple of values in the $m$ output gates of $C$ after the computation.

    The \emph{depth} of $C$ (written $\depth(C)$) is the length of the longest path from an input gate to an output gate in $C$ and the \emph{size} of $C$ (written $\size(C)$) is the number of gates in $C$.
%
    For a gate $g$ in $C$, we will write $\depth(g)$ to denote the length of the longest path from an input gate to $g$ in $C$.
    \label{def:circuit}
\end{definition}

\begin{remark} \label{rem:simu}
    In the context of the circuits over $\R^k$ that we have just defined, a circuit over $\R$ is just a circuit over $\R^1$.
    Note that for circuits over $\R^1$, projection gates are just identity gates and can therefore be omitted.
    It should also be noted that for any fixed $k \in \N$, circuits over $\R^k$ and circuits over $\R$ can easily simulate each other. 
    This result is discussed more extensively in Section \ref{app:prelim} of the Appendix.\looseness=-1
\end{remark}

Since an arithmetic circuit itself can only compute a function with a fixed number of arguments, we extend this definition to families of arithmetic circuits in a natural way.



\begin{definition}
    An \emph{$\R^k$-arithmetic circuit family} $\C$ is a sequence $(C_{n})_{n \in \N}$ of circuits, where each circuit $C_{n}$ has exactly $n$ input gates. 
    Its $\depth$ and $\size$ are functions mapping natural numbers $n$ to $\depth(C_n)$ and $\size(C_n)$, respectively.


    An arithmetic circuit family $\C = (C_{n})_{n \in \N}$ computes the function $f_\C \colon (\R^k)^* \to (\R^k)^*$ defined as 
    \(
        f_\C(\ol{x}) \coloneqq f_{C_{\lvert \ol{x} \rvert}}(\ol{x}).
    \)
    \label{def:circuit_family}
\end{definition}

\begin{remark}
    Since a circuit family is an infinite sequence of circuits, one would be hard pressed to consider such a family a (finite) algorithm.
    Such models of computation, where a different instance is needed for any different input length, is called \emph{non-uniform}.
    In order to deem a circuit family $(C_n)_{n \in \N}$ to be an algorithm, a frequent requirement is the existence of an algorithm which, when given $n$ as an input, outputs the circuit $C_n$.
    Circuit families, for which such an algorithm exists are called \emph{uniform} circuit families.
    For more details on circuit uniformity, see e.g. \cite{DBLP:books/daglib/0097931} and for more on uniformity with respect to real computation, see \cite{ComplexityAndRealComputation}.
\end{remark}

\begin{definition}
    \label{def:fnc_class}
    For any $k \in \N$ and any two functions $s, d \colon \N \to \N$, $\SDk(s, d)$ is the class of all functions $(\R^k)^* \to (\R^k)^*$ that are computable by arithmetic circuit families of size in $\bO(s)$ and depth in $\bO(d)$.
    Let $\mathcal{A}$ be a set of functions of the form $f \colon \R^k \to \R^k$. 
    Then $\SDk(s, d)[\mathcal{A}]$ is the class of functions computable by arithmetic circuits with the same constraints, but with additional gate types $g_f$, with indegree and outdegree $1$ that compute $f$, for each $f\in\mathcal{A}$.  
%
\end{definition}
    We call all classes of the form $\SDk(s,d)[\mathcal{A}]$ (and their subclasses) \emph{circuit function classes}.
As usual in circuit complexity, an \emph{$\mathfrak{F}$-circuit family}, where $\mathfrak{F}$ is a class of functions, will denote a circuit family that computes a function in $\mathfrak{F}$.

We restrict this work to versions of one specific circuit function class.
\begin{definition}
    \label{def:fac0_A}
    \faco is the class of all functions $f\colon (\R^k)^* \to (\R^k)^*$ that can be computed by arithmetic circuit families $(C_n)_{n \in \N}$ of constant depth and polynomial size, i.e., $\faco = \SDk(n^{\bO(1)}, 1)$.
\end{definition}

By Definition \ref{def:fnc_class}, \facoA-circuit families are \faco-circuit families where the circuits are expanded by additional gate types for all functions in $\mathcal{A}$.

A node in a graph neural network aggregates the values of its neighbors in each step, irrespective of their order, and then combines them with its own previous value.
In order for our circuits to mimic this behaviour, we impose a symmetry condition on them.
We require the functions they compute to be \emph{tail-symmetric}, i.e., to only be able to single out their first argument.


\begin{definition}
    \label{def:tail_symmetric_fnc}
   A function $f \colon (\R^k)^* \to (\R^k)^*$ is \emph{tail-symmetric}, if
    \(
        f(\ol{x}_1, \dots, \ol{x}_n) = f(\ol{x}_1, \pi(\ol{x}_2, \dots, \ol{x}_n))
    \)
    for all permutations $\pi$.
\end{definition}

\begin{remark}
    One way to construct a tail-symmetric function $f$ is to take a binary function $g$ and a fully symmetric function $h$ of arbitrary arity and compose them: $f(x_1, \dots, x_n) \coloneqq g(x_1, h(x_2, \dots, x_n))$.
    Note that this is precisely the notion required for the aggregate-combine step in AC-GNNs as per Definition~\ref{def:ac_gnn}.
\end{remark}

 An arithmetic circuit $C$ is \emph{tail-symmetric} if $f_C$ is tail-symmetric.
In the sequel, we will consider families of tail-symmetric functions computed by some circuit family $\mathcal{C}$. 
In this setting, the circuit family computes one unique function $f_n$ for each arity $n\in\mathbb{N}$. 
We will then write $f_\mathcal{C}\left(\ol{x}_1, \lms \ol{x}_2, \dots, \ol{x}_n \rms \right)$ to denote $f_{n}( \ol{x}_1, \dots, \ol{x}_n)$, when $\lms \ol{x}_2, \dots, \ol{x}_n \rms$ is a multiset of cardinality $n-1$.

\begin{definition}
    \label{def:fac0}
    Let $\mathfrak{F}$ be a circuit function class.
    We denote by $t\mathfrak{F}$ the class which contains all functions of $\mathfrak{F}$ that are tail-symmetric.
\end{definition}





\section{Graph Neural Networks using Circuits}
\label{main}

We now define the main computation model of interest for this paper, namely circuit graph neural networks (C-GNNs).
C-GNNs work similarly to plain AC graph neural networks. 
Instead of using the concept of a two-step aggregate and combine computation via which the new feature vector of a node is computed, C-GNNs permit that this computation is done by an arithmetic circuit (with particular resource bounds).
This allows us to classify GNNs based on the complexity of their internal computing functions.

\subsection{Model of Computation}
\label{sec:model_of_comp}
A \emph{basis} of a C-GNN is a set of functions of a specific circuit function class 
along with a set of activation functions.
The respective networks will be based on these functions.

\begin{definition}
    Let $\mathcal{S}$ be a non-empty set of 
    functions from $(\R^k)^*$ to $(\R^k)^*$
    and let $\mathcal{A}$ be a non-empty set of activation functions of dimension $k$. 
    Then we call the set $\mathcal{S} \times \mathcal{A}$ a \emph{C-GNN-basis} of dimension $k$.\looseness=-1
\end{definition}

C-GNNs of a particular basis essentially consist of a fixed depth and a function assigning circuit families and activation functions from its basis to its different layers.

\begin{definition}
    Let $B=\mathcal{S} \times \mathcal{A}$ be a C-GNN-basis of dimension $k$. 
    A \emph{circuit graph neural network (of dimension $k$)} of depth $d \in \N$ is a function $\NN \colon [d] \to B$.
If all functions in $\mathcal{S}$ belong to a circuit function class $\mathfrak{F}$, we also say that $\NN$ is a $(\mathfrak{F}, \mathcal{A})$-GNN.\looseness=-1
\end{definition}

\begin{definition}
\label{def:cgnn_fnc}
    Let $\NN$ be a C-GNN of depth $d$.
    The function $f_\NN \colon \Graph_k \to \Graph_k$ computed by $\NN$ is defined as follows: 

    Let $\lgraph = (V, E, g_V)$ be a labeled graph.
    The labeled graph $\lgraph'$ computed by $f_\NN(\lgraph)$ has the same structure as $\lgraph$, however, its nodes have different feature vectors.
    That is $f_\NN(\lgraph) = \lgraph'$, where $\lgraph' = (V, E, h_V)$ and $h_V$ is defined inductively as follows:
    \begin{align*}
       h_V^{(0)}(w) &\coloneqq~  g_V(w)\\
        h_V^{(i)}(w) &\coloneqq~  \actfct^{(i)} \left( f_{\mathcal{C}^{(i)}}\left(h_V^{(i-1)}(w), M \right) \right), \text{ with } M = \lms h_V^{(i-1)}(u) \mid u \in N_{\lgraph}(w) \rms
    \end{align*}
    where $\NN(i)=\left(\mathcal{C}^{(i)},\actfct ^{(i)}\right)$.
    Finally 
    \(
        h_V(w) \coloneqq  h_V^{(d)}(w).
    \)
\end{definition}

The C-GNN model that we have just defined computes functions from labeled graphs to labeled graphs (with the same graph structure).

Our results are stated for C-GNNs using functions from versions of the function class $\faco$ as defined in Definitions~\ref{def:fac0_A} and \ref{def:fac0}.

\begin{remark}
    \label{rem:relation_ac_gnn_c_gnn}
    We can relate our C-GNN model to the continuous computation of traditional AC-GNNs.
    For any AC-GNN following Definition~\ref{def:ac_gnn} with activation functions $\mathcal{A}$ and where the aggregation functions are computable by $\faco$-circuit families, there exists a $(\faco, \mathcal{A})$-GNN with a constant number of layers which computes the same function, omitting only the functionality of the classification function.
    This holds for the common aggregation functions like the sum, product or mean.
\end{remark}

\subsection{Simulating C-GNNs with arithmetic circuits}
\label{sec:sim_cgnn_with_circ}

\begin{definition}
    Let $\lgraph = (V, E, g_V)$ be a labeled graph, $n \coloneqq \lvert V \rvert$, and
    $M=\text{adj}(\lgraph)$ be the adjacency matrix of $(V, E)$, where the columns are ordered in accordance with the ordering of $V$.
    We write $\enc{M}$ to denote the encoding of $M$ as the $n^2$ matrix entries written using vectors $(0,\dots, 0)$ and $(1,\dots, 1)$ in $\R^k$, ordered in a row wise fashion. 
    We denote them by $\ol{m}_{ij}$ encoding the matrix entry $m_{ij}$.
    We write $\enc{\lgraph}$ to denote the encoding of $\lgraph$ as a tuple of real valued vectors, such that $\enc{\lgraph}= \left(\enc{M}, \text{vec}(\lgraph)\right) \in (\R^k)^{n^2+n}$, which consists of the encoding of $M$ followed by $\text{vec}(\lgraph)$, the feature vectors of $\lgraph$.
    The feature vectors $\text{vec}(\lgraph)$ are $g_V(v)$ for all $v \in V$ and ordered like $V$.

    
\end{definition}

\begin{theorem*}
    Let $\NN$ be a $(t\facoA, \{\mathrm{id}\})$-GNN.
    Then there exists an $\facoA$-circuit family $\mathcal{C}$, such that for all labeled graphs $\lgraph$ the following holds:
        $f_\NN(\lgraph) = \mathfrak{G'} \text{ iff } \mathcal{C}(\enc{\lgraph}) = \enc{\mathfrak{G'}}$,
    where $\enc{\lgraph} = (\enc{\text{adj}(\lgraph)}, \enc{\text{vec}(\lgraph)})$. 
    \label{thm:cgnn_to_circ}
\end{theorem*}
\begin{proof}
    
    We essentially roll out the given $(t\facoA, \{\mathrm{id}\})$-GNN $\NN$ by using a circuit to simulate each individual layer of $\NN$ and concatenating these circuits to then simulate $\NN$.
    Since the computation in each node of $\NN$ can be performed by a circuit of an $\facoA$-family, simulating this computation can be done without an issue.
    Similarly, connecting the simulated nodes and layers can be done relatively simply, and thus the resulting circuit family is an $\facoA$-family.
\end{proof}

\begin{remark}    
    {Note that this proof preserves uniformity: if the sequence of C-GNNs is uniform, so is the circuit family.
    Furthermore, for any set of activation functions $\mathcal{A}$, $(t\facoA, \{\mathrm{id}\})$-GNNs are at least as powerful as $(t\faco, \{\mathrm{id}\} \cup \mathcal{A})$-GNNs and $(t\faco, \{\mathrm{id}\})$-GNNs.
    The latter do not have access to the functions of $\mathcal{A}$ at all and $(t\faco, \{\mathrm{id}\} \cup \mathcal{A})$-GNNs are more restricted in the use of the functions than $(t\facoA, \{\mathrm{id}\})$-GNNs.}
\end{remark}

\subsection{Simulating arithmetic circuits with C-GNNs}
\label{sec:sim_circ_with_cgnn}

In order to simulate $\faco$-circuit families using C-GNNs, we utilize the following normal form. 

\begin{definition}
    A circuit $C$ is in \emph{path-length normal form} if every path from an input to an output gate has the same length and every non-input and non-output gate has exactly one successor.
\end{definition}

\begin{lemma}
    Let $\mathcal{C}$ be an $\faco$-circuit family.
    Then for every circuit $C \in \mathcal{C}$ there exists a circuit $C'$ of the same depth in path-length normal form such that \(f_{C'}(\ol{x}) = f_{C}(\ol{x})\), for all $\ol{x} \in (\R^k)^*$.
\end{lemma}
\begin{proof}
    We describe a procedure that transforms $C$ to $C'$.
    If a gate $h$ in $C$ has more than one successor, the subgraph consisting of $h$ and all paths from input gates to $h$ are copied in $C'$ for each successor of $h$. 
    These copies are then used as the respective input for those successors which results in every gate in $C'$ having only one successor.
    This is done iteratively, where in each step only gates in one depth layer of the circuit are modified, starting with depth 1, the gates closest to the input gates. 
    Each step in this procedure incurs a polynomial overhead in size. 
    Since $C$ is of constant depth there are only constantly many iteration steps and thus $C'$ remains polynomial in size. 
    \looseness=-1
     
    Let $\ell_{max}=\depth(C)(=\depth(C'))$ be the length of the longest path from an input to an output node in $C$.
    For every path of length $\ell$ which is shorter than $\ell_{max}$, we add to $C'$ $\ell_{max}-\ell$ unary addition gates directly after the input.
    The depth of the resulting circuit $C'$ equals the one of $C$, as all changes to path lengths are made with this upper bound in mind.
    Since all changes made preserve the computed function and the produced $C'$ is in path-length normal form, the claim follows.
\end{proof}

Since C-GNNs get labeled graphs as their input, while arithmetic circuits typically work on vectors over $\R^k$, we show next how to represent a circuit as an input to a C-GNN.

\begin{definition} \label{def:circ_to_cgnn_feature_graph_construction}
    If $C$ is a circuit with input $\ol{x}$ whose gates are additionally numbered with unique elements of $\R$, let \lgc{C} be the labeled graph $(G, g_{V})$, where $G=(V, E)$ is the underlying graph of the circuit and the function $g_{V}$ is defined as follows.
    The feature vector of a graph node corresponding to an input gate is the corresponding input value.
    Each other node gets its own number according to the gate numbering of the circuit as its feature vector. 
\end{definition}

The following theorems describe how C-GNNs are able to simulate arithmetic circuits. 
The proofs are given for $\R$ circuits, but as mentioned in Remark \ref{rem:simu} and proved in Corollary \ref{col:Rk_R_circ_equiv} families of constant depth and polynomial size $\R^k$ and $\R$ circuits can simulate each other with constant increase/decrease in size, so they also hold for $\R^k$ circuits.

\begin{theorem*}
\label{thm:circ-gnn-without-fnc}
    Let $\mathcal{C} = (C_n)_{n \in \N}$ be an $\mathrm{FAC}^0_{\R}$-circuit family.
    Then there exists a $(t\mathrm{FAC}^0_{\R}, \{\mathrm{id}\})$-GNN $\NN$, such that for all $n \in \N$ and $\ol{x} \in \R^n$,
    the feature vectors of the nodes of $f_{\NN}(\lgc{C_n})$ corresponding to output gates in $C_n$ are exactly 
    the components of $f_{C_n}(\ol{x})$ in the same order. 
    The number of layers in $\NN$ is equal to the depth of $\mathcal{C}$.
    %

\end{theorem*}
\begin{proof}
    The general idea is to use the graphs of the circuits from $\mathcal{C}$ as input graphs for a C-GNN and then simulate the gates layerwise. 
    The gates of $\mathcal{C}$ are numbered and the numbering is represented in the feature vectors for each node.
    In each layer circuits of the C-GNN use a lookup table to decide which operation is applied to a feature vector (see Lemma \ref{lem:chiA}).
    An example illustrating the proof idea is given in Figure \ref{fig:circ_to_cgnn_no_fnc} and Table \ref{tbl:circ_to_cgnn_without_fnc}. For simplicity reasons an example omitting any special cases is chosen.\looseness=-1
\end{proof}

\begin{figure}[htb]
    \begin{subfigure}[b]{0.45\textwidth}
        \centering
        \begin{tikzpicture}[
            gate/.style={circle, draw, minimum size =0.5cm}
            ]
            \node[]         (vin1)                                      {$6$};
            \node[]         (vin2)  [right=of vin1]                     {$9$};
            \node[]         (vin3)  [right=of vin2]                     {$5$};
            \node[gate]     (v+)    at ($(vin1)!0.5!(vin2)+(0,-1)$)   [label=left:{$nr(g_+)=1$}]      {$+$};
            \node[gate]     (vx)    at ($(v+)!0.5!(vin3)+(0,-1.5)$)   [label=left:{$nr(g_\times)=2$}] {$\times$};
            \node           (vout)  [below=0.5cm of vx]                 {$\vdots$};
            \draw[-stealth] (vin1)    --  (v+);
            \draw[-stealth] (vin2)    --  (v+);
            \draw[-stealth] (vin3)    --  (vx);
            \draw[-stealth] (v+)      --  (vx);
            \draw[-stealth] (vx)      --  (vout);
        \end{tikzpicture}
        \caption{Underlying circuit, to be simulated by the C-GNN ($nr(\ast)$ denotes the unique number of the gate).}
    \end{subfigure}\hspace{1em}%
    \begin{subfigure}[b]{0.45\textwidth}
        \centering
        \begin{tikzpicture}[
            gate/.style={circle, draw, minimum size =0.5cm}
            ]
            \node[gate]     (vin1)  [label=above:{$\ol{v}^{(1)}_{\mathrm{in}_1}=6$}]                      {$v_{\mathrm{in}_1}$};
            \node[gate]     (vin2)  [label=above:{$\ol{v}^{(1)}_{\mathrm{in}_2}=9$},right=4em of vin1]    {$v_{\mathrm{in}_2}$};
            \node[gate]     (vin3)  [label=above:{$\ol{v}^{(1)}_{\mathrm{in}_3}=5$},right=4em of vin2]    {$v_{\mathrm{in}_3}$};
            \node[gate]     (v+)    at ($(vin1)!0.5!(vin2)+(0,-1)$)  [label=left:{$\ol{v}^{(1)}_+=1$}]        {$v_+$};
            \node[gate]     (vx)    at ($(v+)!0.5!(vin3)+(0,-1.5)$)  [label=left:{$\ol{v}^{(1)}_\times = 2$}] {$v_\times$};
            \node           (vout)  [below=0.5cm of vx]                                                     {$\vdots$};
            \draw[-] (vin1)     --  (v+);
            \draw[-] (vin2)     --  (v+);
            \draw[-] (vin3)     --  (vx);
            \draw[-] (v+)       --  (vx);
            \draw[-] (vx)       --  (vout);
        \end{tikzpicture}
        \caption{The labeled graph and initial feature vectors of the C-GNN ($v_g^{(i)}$ denotes the feature vector of a gate at layer $i$).}
    \end{subfigure}
    \caption{Example illustrating the proof of Theorem~\ref{thm:circ-gnn-without-fnc}. \protect \footnotemark}

    \captionof{table}{Example illustrating the proof of Theorem~\ref{thm:circ-gnn-without-fnc}: The values of the feature vectors during the computation of the C-GNN.}
    \label{tbl:circ_to_cgnn_without_fnc}
    \vskip 0.15in
    \begin{center}
    \begin{small}
    \begin{sc}
    \begin{tabular}{cccccc}
    \toprule
    Layer & $\ol{v}_{\mathrm{in}_1}$ & $\ol{v}_{\mathrm{in}_2}$ & $\ol{v}_{\mathrm{in}_3}$ & $\ol{v}_+$ &$\ol{v}_\times$\\
    \midrule
    1   & $6$ & $9$ & $5$ & $1$        & $2$ \\
    2   & $6$ & $9$ & $5$ & $6+9=15$   & $2$ \\
    3   & $6$ & $9$ & $5$ & $15$       & $5\times 15 = 75$ \\
    \bottomrule
    \end{tabular}
    \end{sc}
    \end{small}
    \end{center}
    \vskip -0.1in
    \label{fig:circ_to_cgnn_no_fnc}
\end{figure}

Since we have shown in Section \ref{sec:sim_cgnn_with_circ} that C-GNNs can be simulated by $\facoA$-circuit families we now examine the converse direction: simulating these families with C-GNNs.
\begin{theorem*} \label{thm:circ-cgnn-fnc-in-circ}

    Let $\mathcal{C} = (C_n)_{n \in \N}$ be an $\mathrm{FAC}^0_{\R}[\mathcal{A}]$-circuit family.
    Then there exists a $(t\mathrm{FAC}^0_{\R}[\mathcal{A}], \{\mathrm{id}\})$-GNN $\NN$, such that for all $n \in \N$ and $\ol{x} \in \R^n$,
    the feature vectors of the nodes of $f_{\NN}\left(\lgc{C_n}\right)$ corresponding to output gates in $C_n$ are exactly 
    the components of $f_{C_n}(\ol{x})$ in the same order.
    The number of layers in $\NN$ is equal to the depth of $\mathcal{C}$.
\end{theorem*}
\begin{proof}
    The proof follows the same concept as the proof of Theorem \ref{thm:circ-gnn-without-fnc}.
    With the exception that the internal circuits of the C-GNN now have access to the same set of functions $\mathcal{A}$ in the form of gates that compute them.
    Those are then used to simulate gates from $\mathcal{C}$ that compute functions of $\mathcal{A}$.\looseness=-1
\end{proof}


In the previous two theorems, we have shown that C-GNNs can simulate $\facoR$-circuit families as is and that they can simulate $\facoAR$-circuit families if they have arbitrary access to the functions in $\mathcal{A}$.
If we consider $\mathcal{A}$ to consist of activation functions, though, it would be more natural to only permit them as such in our C-GNN model.
In the following we will introduce a structural restriction of circuits, which will allow us to capture a subset of $\facoAR$ using C-GNNs that have access to $\mathcal{A}$ only as activation functions.

\begin{definition}
    A circuit $C$ is said to be in \emph{function-layer form} if it is in path-length normal form and for each depth $d \leq \depth(C)$ all gates of $C$ at depth $d$ have the same gate type.
\end{definition}

\begin{definition}
    Let $\mathfrak{F}$ be a circuit function class.
    We denote by $f\mathfrak{F}$ the class which contains all functions of $\mathfrak{F}$ that can be computed by circuit families where all circuits are in function-layer form.
\end{definition}

We also need a small restriction on the functions we will permit as activation functions.

\begin{definition}
    A computable function $f \colon \R^k \to \R^k$ is \emph{countably injective} if there is a countably infinite set $S \subseteq \R^k$ such that $f\!\! \restriction_S$ is injective.
    We also say that $f$ is countably injective on $S$. \looseness=-1
\end{definition}

\begin{remark}
    The commonly used activation functions $\mathrm{ReLU}(x) = \max(0,x)$, $\sigma(x) = \frac{1}{1+ \mathrm{e}^{-x}}$ and $\mathrm{tanh}(x) = \frac{\mathrm{e}^x - \mathrm{e}^{-x}}{\mathrm{e}^x+\mathrm{e}^{-x}}$ are countably injective.
\end{remark}

\begin{theorem*} \label{thm:circ-gnn-fnc}
    Let $\mathcal{A}$ be a set of countably injective activation functions and let $\mathcal{C} = (C_n)_{n \in \N}$ be a $f\mathrm{FAC}^0_{\R}[\mathcal{A}]$-circuit family. 
    Then there exists a $(t\mathrm{FAC}^0_{\R}, \mathcal{A} \cup \{ \mathrm{id}\})$-GNN $\NN$, such that for all $n \in \N$ and $\ol{x} \in \R^n$,
    the feature vectors of the nodes of $f_{\NN}\left(\lgc{C_n}\right)$ corresponding to output gates in $C_n$ are exactly 
    the components of $f_{C_n}(\ol{x})$ in the same order.
    The number of layers in $\NN$ is equal to the depth of $\mathcal{C}$.\looseness=-1
\end{theorem*}

\begin{figure}[htb]
    \begin{subfigure}[b]{0.45\textwidth}
        \centering
    \begin{tikzpicture}[
        gate/.style={circle, draw, minimum size =0.5cm}
        ]
        \node[]         (vin1)                                      {$6$};
        \node[]         (vin2)  [right=of vin1]                     {$9$};
        \node[gate]     (vsigma)[below=0.5cm of vin1, label=left:{$nr(g_ \sigma)=1$}]      {$\sigma$};
        \node[gate]     (vsigma')[below=0.5 cm of vin2, label=right:{$nr(g'_ \sigma)=2$}]      {$\sigma$};
        \node[gate]     (v+)    at ($(vsigma)!0.5!(vsigma')+(0,-1)$)   [label=left:{$nr(g_+)=3$}]      {$+$};
        \node     (vout)  [below=0.5cm of v+]  {$\vdots$};
        \draw[-stealth] (vin1) -- (vsigma);
        \draw[-stealth] (vin2) -- (vsigma');
        \draw[-stealth] (vsigma)   -- (v+);
        \draw[-stealth] (vsigma') -- (v+);
        \draw[-stealth] (v+) -- (vout);
    \end{tikzpicture}
    \caption{Underlying circuit, to be simulated by the C-GNN ($nr(\ast)$ denotes the unique number of the gate).}
    \end{subfigure}\hspace{1em}%
    \begin{subfigure}[b]{0.45\textwidth}
        \centering
        \begin{tikzpicture}[
            gate/.style={circle, draw, minimum size =0.5cm}
            ]
            \node[gate]     (vin1)      [label=left:{$\ol{v}^{(1)}_{\mathrm{in}_1}=6$}]                           {$v_{\mathrm{in}_1}$};
            \node[gate]     (vin2)      [label=right:{$\ol{v}^{(1)}_{\mathrm{in}_2}=9$},right=4em of vin1]        {$v_{\mathrm{in}_2}$};
            \node[gate]     (vsigma)    [below=0.5cm of vin1, label=left:{$\ol{v}^{(1)}_+=1$}]                          {$v_\sigma$};
            \node[gate]     (vsigma')   [below=0.5cm of vin2, label=right:{$\ol{v}'^{(1)}_+=2$}]                        {$v_\sigma'$};
            \node[gate]     (v+)    at ($(vsigma)!0.5!(vsigma')+(0,-1)$)  [label=left:{$\ol{v}^{(1)}_+=3$}]     {$v_+$};
            \node     (vout)  [below=0.5cm of v+]  {$\vdots$};
            \draw[-] (vin1) -- (vsigma);
            \draw[-] (vin2) -- (vsigma');
        \draw[-] (vsigma') -- (v+);
        \draw[-] (vsigma) -- (v+);
        \draw[-] (v+) -- (vout);
        \end{tikzpicture}
        \caption{The labeled graph and initial feature vectors of the C-GNN ($v_g^{(i)}$ denotes the feature vector of a gate at layer $i$).}
    \end{subfigure}
    \caption{Example illustrating the proof of Theorem~\ref{thm:circ-gnn-fnc}. \protect \footnotemark[\value{footnote}]}
    \captionof{table}{Example illustrating the proof of Theorem~\ref{thm:circ-gnn-fnc}: The values of the feature vectors during the computation of the C-GNN.}
    \label{tbl:circ_to_cgnn_activation_fnc}
    \vskip 0.15in
    \begin{center}
    \begin{small}
    \begin{sc}
    \begin{tabular}{cccccc}
    \toprule
    Layer & $\ol{v}_{\mathrm{in}_1}$ & $\ol{v}_{\mathrm{in}_2}$ &  $\ol{v}_\sigma$ & $\ol{v}'_\sigma$ & $\ol{v}_+$\\
    \midrule
    1       & $6$           & $9$           & $1$           & $2$           & $3$                         \\
    2       & $\sigma(6)$   & $\sigma(9)$   & $\sigma(1)$   & $\sigma(9)$   & $\sigma(\sigma^{-1}(3))=3$  \\
    3       & $\sigma(6)$   & $\sigma(9)$   & $\sigma(1)$   & $\sigma(9)$   & $\sigma(6)+\sigma(9)$       \\
    \bottomrule
    \label{tbl:circ_to_cgnn_fnc}
    \end{tabular}
    \end{sc}
    \end{small}
    \end{center}
    \vskip -0.1in
    \label{fig:circ_to_cgnn_fnc}
\end{figure}
\begin{proof}
    Contrary to the proof of Theorem \ref{thm:circ-cgnn-fnc-in-circ} the gates in $\mathcal{C}$ computing a function $\sigma \in \mathcal{A}$ are simulated in the C-GNN by using the respective function as an activation function applied to all nodes. 
    To distinguish between nodes where this function should and should not be applied we assign $\sigma^{-1}$ of the respective feature vectors of nodes where we do not want to change the value.
    This is done prior to the application of the activation function.
    Therefore the function needs to be injective on some countable set.
    An example illustrating the proof idea is given in Figure \ref{fig:circ_to_cgnn_fnc} and Table \ref{tbl:circ_to_cgnn_fnc}.
    As in the proof of Theorem \ref{thm:circ-gnn-fnc} an example omitting any special cases is chosen.
\end{proof}

\footnotetext{{For readability purposes the one-dimensional vectors in the C-GNN are written as real numbers.}}

\section{Conclusion}
\label{sec:conclusion}

In this paper we showed a correspondence between arithmetic circuits and a generalization of graph neural networks using circuits.
However, some restrictions needed to be imposed on the particular circuits used in our constructions.
In particular, the circuits used in our C-GNNs need to be tail symmetric, and in Theorem~\ref{thm:circ-gnn-fnc} the corresponding $\facoAR$-circuit family needs to be in function-layer form as well.
An interesting avenue for further research is the question, whether those restrictions can be made ``on both sides in our simulations''.
This means, in particular, to ask if $(t\faco, \{id\} \cup \mathcal{A})$-GNNs can be simulated by $f\facoA$ families and, if $(tf\faco, \{id\} \cup \mathcal{A})$-GNNs and $f\facoA$ can simulate each other.
Another direction of further work is to study whether by relating our C-GNNs to variants of so-called VVc-GNNs of \citet{DBLP:conf/nips/SatoYK19}, we may drop the assumption of tail-symmetry.




We already mentioned the issue of uniformity. 
If the sequence of GNNs, one for each graph size, is uniform in the sense that there is an algorithm that, given the size of the graph as an input, outputs the GNN responsible for that size, then the sequence of arithmetic circuits will also be uniform, because our simulation proof explicitly shows how to construct the circuit; and vice versa. 
In future work, this should be made more precise. 
In the case of Boolean circuits, logtime-uniformity (U$_{E°*}$-uniformity) has become the standard requirement \citep{DBLP:books/daglib/0097931}.
What is the corresponding precise uniformity notion for GNNs? 

In the introduction we mentioned the complexity of the problem of training neural networks. 
We studied the expressiveness (or computational power) of neural networks. 
The higher the expressiveness of a network is, the more complicated the training process will be; however, the question to decide if a trained network of given quality exists might become easier to decide.
Can this be made precise? Is there a formal result connecting complexity of training and network expressiveness?

Further investigations are required to obtain practical implications of our theoretical results.
Once we fix the architecture of our C-GNNs (i.e., a circuit function class for the GNN nodes), we obtain a specific circuit class that characterizes the computational power of the C-GNN model which can be then rigorously studied. 
The question of what can be computed by arithmetic circuits using activation functions is related to the question of which functions can be computed using other functions, or in other words, which functions are "more complex" than others. 
We are only aware of very little work in this area. 
Results include a PTIME upper bound (in the so-called \emph{BSS}-model of computation) for the arithmetic circuit complexity class $\mathrm{NC}_\R$ that uses the sign function \cite{DBLP:journals/jc/Cucker92}.

As it currently stands, there does not seem to be meaningful experiments that could be run for our model of computation. 
The problem is that not much is known about the computational power of the circuit classes we utilize. 
Further research into functions computable by practically implementable arithmetic circuits of different sizes and depths would be required. 
This could e.g. be done by investigating the Boolean parts of different complexity classes defined by families of arithmetic circuits, i.e., the classes when restricted to Boolean inputs. 
This research would pinpoint properties that would be provably outside of the capabilities of particular GNN models, whose learnability could be then tested in practice. 
We believe that our results will motivate research for these circuit classes.

\citet{DBLP:journals/tacl/MerrillSS22,DBLP:journals/corr/abs-2207-00729}
have studied transformer networks from a computational complexity perspective and obtained a $\TCzero$ upper bound, similar to the result for GNNs by~\cite{DBLP:conf/lics/Grohe23}. 
We think it is worthwhile to study the computational power of transformers from the point of view of computation over the reals, similar to what we have done in this paper for GNNs.  

 \section*{Acknowledgements}
 The first and the fourth author were partially supported by the DFG grant VI 1045-1/1.

\bibliography{ref}

\newpage
\appendix
\onecolumn
\section{Appendix}
The Appendix is organized as follows: Section \ref{app:prelim} extends the results on the correspondence of circuits over $\R^k$ and $\R$. Section \ref{sec:app_1} goes into the concepts of Section \ref{sec:model_of_comp} and contains a proof for the remark concerning the connection between AC-GNNs and C-GNNs. The other Sections \ref{sec:app_2} and \ref{sec:app_3} contain the full proofs for the theorems of Sections \ref{sec:sim_cgnn_with_circ} and \ref{sec:sim_circ_with_cgnn} in the paper, as well as additional required results.
\subsection{Proofs of Section \ref{sec:prelim}} \label{app:prelim}
The following Lemma shows in depth that $\R^k$ circuits can be simulated by $\R$ circuits.
\begin{lemma}\label{lem:rkcirc_to_rcirc}
    Let $k, n, m \in \N$ be arbitrary.
    For each arithmetic circuit $C$ over $\R^k$ with $n$ input and $m$ output gates, there exists an arithmetic circuit $C'$ over $\R$ with $n \cdot k$ input and $m \cdot k$ output gates, such that $\size(C') \leq k \cdot \size(C)$, $\depth(C') \leq \depth(C)$ and for all $\ol{x}, \ol{y} \in \R^k$
    \[f_C\left((x_{11}, \dots, x_{1k}), \dots, (x_{n1}, \dots, x_{nk})\right) =((y_{11}, \dots, y_{1k}), \dots, (y_{m1}, \dots, y_{nk}))\]
    if and only if 
    \[f_{C'}(x_{11}, \dots, x_{1k}, \dots, x_{n1}, \dots, x_{nk}) =(y_{11}, \dots, y_{1k}, \dots, y_{m1}, \dots, y_{mk}).\]    
\end{lemma}

\begin{proof}
    Each componentwise operation is replaced by the respective individual real operations. 
    Projections are realized by directly working with the projected value and disregarding the values corresponding to the other elements of the vectors and replacing them by constant $0$ gates. 
    This way, there are $k$ arithmetic gates in $C'$ for each arithmetic gate in $C$ and the depth of $C$ is not increased.
\end{proof}

Similarly, $\R$ circuits can be simulated by $\R^k$ circuits.

\begin{lemma}\label{lem:rcirc_to_rkcirc}
    Let $k, n, m \in \N$ be arbitrary.
    For each arithmetic circuit $C$ over $\R$ with $n \cdot k$ input and $m \cdot k$ output gates, there exists an arithmetic circuit $C'$ over $\R^k$ with $n$ input and $m$ output gates, such that $\size(C') \leq k \cdot (n + m) + m + \size(C)$, $\depth(C') \leq \depth(C) + 3$ and for all $\ol{x}, \ol{y} \in \R^k$
    \[        f_{C}(x_{11}, \dots, x_{1k}, \dots, x_{n1}, \dots, x_{nk}) =
         (y_{11}, \dots, y_{1k}, \dots, y_{m1}, \dots, y_{mk})
    \]
    if and only if 
    \[f_{C'}((x_{11}, \dots, x_{1k}), \dots, (x_{n1}, \dots, x_{nk})) =
        ((y_{11}, \dots, y_{1k}), \dots, (y_{m1}, \dots, y_{mk})).\]
\end{lemma}

\begin{proof}
    First project all tuple elements to the first component of each tuple (i.e. $(x_1, \dots, x_k) \mapsto (x_1, 0, \dots, 0), (x_2, 0, \dots, 0), \dots, (x_k, 0, \dots, 0)$), then compute as in the $\R^1$ circuit and finally project back when reaching the output gates.
    Projecting from the input gates requires $k$ projection gates in $C'$ for each input gate in $C$ and projecting back to the output gates requires $k$ projection gates and one addition gate, to put all projected values back into one vector.
    This results in a total overhead in size of $k \cdot (n + m) + m$ and a total overhead in depth of $3$.
\end{proof}

\begin{corollary}
\label{col:Rk_R_circ_equiv}
    For each arithmetic circuit family $\C$ over $\R$ of polynomial size and constant depth, there exists an arithmetic circuit family $\C'$ over $\R^k$ of polynomial size and constant depth such that for all $\ol{x}, \ol{y} \in \R^k$
    \[f_{\C}(x_{11}, \dots, x_{1k}, \dots, x_{n1}, \dots, x_{nk}) =
        (y_{11}, \dots, y_{1k}, \dots, y_{m1}, \dots, y_{mk})\]
    if and only if 
    \[f_{\C'}\left((x_{11}, \dots, x_{1k}), \dots, (x_{n1}, \dots, x_{nk})\right) =
        ((y_{11}, \dots, y_{1k}), \dots, (y_{m1}, \dots, y_{nk})).\]
    Respectively an analogous family over $\R$ exists for each family over $\R^k$. \looseness=-1
\end{corollary}

\subsection{Proofs of Section \ref{sec:model_of_comp}}
\label{sec:app_1}
As mentioned in Remark \ref{rem:relation_ac_gnn_c_gnn} our considered model of C-GNNs relates to the classical AC-GNNs in the following way. 
The definition of an AC-GNN does not fix the complexity of its aggregate and combine functions, hence to be able to compare the two models we need to restrict both functions to be of a specific circuit function class. This restriction holds for most practically used functions like aggregating via the weighted sum and combining via summation of the previous feature vector and the aggregate, each multiplied by a parameter matrix.  
We now proceed to prove this. \medskip

\begin{customrem}{\ref{rem:relation_ac_gnn_c_gnn}}
    We can relate our C-GNN model to the continuous computation of traditional AC-GNNs.
    For any AC-GNN following Definition~\ref{def:ac_gnn} with activation functions $\mathcal{A}$ where the aggregation functions are computable by $\faco$-circuit families, there exists a $(\faco, \mathcal{A})$-GNN with a constant number of layers which computes the same function, omitting only the functionality of the classification function.
    This holds for the common aggregation functions like the sum, product or mean.
\end{customrem}
\begin{proof}
    Let $\mathcal{D}=(\{ \text{AGG}^{(i)}\}_{i=1}^L, \{ \text{COM}^{(i)} \}_{i=1}^L, \{ \sigma^{(i)} \}_{i=1}^L, \allowbreak\text{CLS})$ be an arbitrary AC-GNN as defined in Definition \ref{def:ac_gnn}. 
    A $(\faco, \mathcal{A})$-GNN computing the same function can be realized by using $\faco$-circuit families which compute the functions in $\{ \text{AGG}^{(i)}\}_{i=1}^L$ and $\{ \text{COM}^{(i)} \}_{i=1}^L$. 
    To compute a permutation-invariant aggregation function a tail-symmetric circuit family $\mathcal{C}_{\text{AGG}}$ computing a function $f_{\mathcal{C}_{\text{AGG}}}=\text{AGG}^{(i)}$ as described in Definition \ref{def:tail_symmetric_fnc} is used.
    The computation of a combine function can also be realised by using $\faco$- circuit families computing a function $f_{\mathcal{C}_{\text{COMB}}}=\text{COMB}^{(i)}$.
    Afterwards the application of the activation function $\sigma^{(i)}$ can be achieved by adding $\sigma^{(i)}$ to the set $\mathcal{A}$ and denoting it as the activation function for layer $i$ which is applied to all feature vectors after the circuit computations of these layers are done.
    As said in the beginning we are not concerned with classification in this paper, so we do not aim to simulate the transition from the continuous to the discrete as performed in the classification function of AC-GNNs.
\end{proof}

\subsection{Proofs of Section \ref{sec:sim_cgnn_with_circ}}
\label{sec:app_2}
The proof of Theorem \ref{thm:cgnn_to_circ} requires us to be able to distinguish the feature vectors of the neighborhood of a given node. 
This can be achieved by essentially zeroing out all non-neighbors as follows.
\begin{lemma}
    \label{lem:nn_nodes_zero}
    For each $k,n,i \in \N$ where $1 \leq i \leq n$ there exists an $\faco$-circuit family $\mathcal{C}$ such that for all $\mathfrak{G} = (V, E, g_V) \in \Graph_k$ with $V = \{v_1, \dots, v_n\}$ it holds that
    \[
        f_\mathcal{C}(\enc{\mathfrak{G}}) = \enc{\text{vec}(\lgraph)'}
    \]
    where $\enc{\mathfrak{G}} = (\enc{\text{adj}(\lgraph)}, \enc{\text{vec}(\lgraph)})$ and for all $1 \leq j \leq n$
    \[
        \text{vec}(\lgraph)'_j = 
        \begin{cases}
            \text{vec}(\lgraph)_j, &\text{ if $\{v_i, v_j\} \in E$} \\
            \ol{0}, &\text{otherwise}
        \end{cases}
    \]
    where $\text{vec}(\lgraph)'_j$ and $\text{vec}(\lgraph)_j$ are the respective $j$th vectors of $\text{vec}(\lgraph)'$ and $\text{vec}(\lgraph)$.
\end{lemma}
\begin{proof}
    We write $v_i$ for the node $v$ which has number $i$ and $\ol{v}_i$ for its feature vector.
    Since the vectors in $\enc{M}$ with $M=\text{adj}(\lgraph)$ are ordered and the vectors $\ol{m}_{ij}$ for $1 \leq j \leq \lvert V \rvert$ correspond to one row in the adjacency matrix, row $\ol{m}_{ij}$ describes all edge relations of $v_i$.
    Given the ordering on $\text{vec}(\lgraph)$ we know that the $i$th vector in $\text{vec}(\lgraph)$ is the vector $\ol{v}_i$.
    We multiply each vector $\ol{m}_{ij}$ with $\ol{v}_j$, resulting in
    \[
    \ol{m}_{ij}\ol{v}_j = \begin{cases}
        \ol{v}_j, &\text{if } (v_i, v_j) \in E\\
        \ol{0}, &\text{otherwise}
    \end{cases}.
    \]
    We continue to work with these values.
    This sets all feature vectors of non-neighboring nodes of $v_i$ to $\ol{0}$.
\end{proof}

We now proceed to prove Theorem~\ref{thm:cgnn_to_circ}.
A simplified illustration of the proof is given in Figure \ref{fig:cgnn_to_circ_idea}.

\begin{customthm}{\ref{thm:cgnn_to_circ}}
    Let $\NN$ be a $(t\facoA, \{\mathrm{id}\})$-GNN. 
    Then there exists an $\facoA$-circuit family $\mathcal{C}$, such that for all labeled graphs $\lgraph$ the following holds:
    \[
        f_\NN(\lgraph) = \mathfrak{G'} \text{ iff } \mathcal{C}(\enc{\lgraph}) = \enc{\mathfrak{G'}}
    \]
    where $\enc{\lgraph} = (\enc{\text{adj}(\lgraph)}, \enc{\text{vec}(\lgraph)})$.
\end{customthm}

\begin{proof}
    We would like to mimic each of the (constantly many) C-GNN layers with a circuit of constant depth.
    To do this we need to retrieve the correct circuit from the circuit families of our C-GNN for each layer $i$ and for each possible size of the neighborhood for a node $v \in V$. 
    These values are bounded by the number of layers $d$ of the C-GNN and the size $n = \lvert V \rvert$ of the input labeled graph $\lgraph = (V, E, g_V)$.

    Since $\lgraph$ has $n$ nodes, our encoding $\enc{\lgraph}$ consists of $n^2$ vectors for the encoding of the adjacency matrix $M$ and the feature vectors of the $n$ nodes. 
    We construct a circuit family $\mathcal{K}$ which consists of a series of circuits for any possible number of nodes $n$. 
    
    While constructing a circuit $K$ from $\mathcal{K}$ we use $\mathcal{C}^{(i)} = \proj[1]{\left(\NN(i)\right)}$ (see Definition \ref{def:cgnn_fnc}) to retrieve the circuit family of layer $i$ of the \CGNN.
    The first observation we make is that if a node $v$ has $k$ neighbors, we need to consider the $k+1$st circuit from the circuit family for the computation, as this circuit gets all of the feature vectors of the neighbors of $v$ and the feature vector of $v$ as input.
    We also observe that, since our circuits are tail-symmetric, the order of the neighbors' feature vectors does not matter.
    We can figure out how many neighbors $v$ has by summing over $v$'s row in the adjacency matrix of $\lgraph$.

    The circuit $K$ works for a node $v$ in the following way.
    When given the initial feature vectors and the encoding of the adjacency matrix $M$ of $\enc{\lgraph}$ as input it first computes the size $k$ of the neighborhood of $v$. 
    It then sets the feature vectors of non-neighboring nodes of $v$ to $\ol{0}$ via the functionality described in Lemma \ref{lem:nn_nodes_zero}.
    Afterwards it orders the feature vectors of all nodes in $\lgraph$ such that the initial order remains the same, but all $\ol{0}$ vectors get pushed to the end.

    We can achieve such an ordering by essentially using the idea of the counting sort algorithm. 
    We start by zeroing the non-neighbors, and then do a pairwise comparison with all nodes in the following way: 
    When we compare two nodes $v_p$ and $v_q$ where $v_p$ is a neighbor of $v$ and $v_q$ is not, then $v_p > v_q$ and if neither or both are neighbors of $v$, then $v_p > v_q$ iff $p > q$.
    We can do these kinds of comparisons, since we know that the comparison only needs to work for all numbers $\leq n-1$ in $\N_0$ and we can thus create a polynomial that works like the sign function for all these values. 
    We then count (i.e. sum up) the number of comparison victories for each node and use equality checks (again, for numbers $\leq n-1$ in $\N_0$) to make sure the right vector is put in the right place.

    Afterwards all circuits $C^{(i)}_j$ from the circuit family $\mathcal{C}^{(i)}$ for $j \in \{1, \dots, n\}$ are simulated in parallel. 
    Since the circuits $C^{(i)}_j$ are tail-symmetric, it does not matter that we changed the order of the vectors.
    Moreover since at most $k+1$ vectors are nonzero after we zeroed the non-neighbors, they are in the first $k+1$ vectors in our new ordering.
    Each of these simulations is followed by an equality check on $j$ and $k+1$.
    The sum is then formed over these which ensures that only the results of the circuit simulation for the correct number of neighbors of $v$ is considered further.
    
    We repeat the described computations for all nodes $v \in V$ for every layer $0\leq i \leq d$, using the resulting feature vectors from the computations before as input to the following ones. 
    Now we have the following:
    \begin{itemize}
        \item The circuit $K$ has constant depth.
        \item The computations for one layer are done in parallel for every node.
        \item The number of layers is a constant given by the C-GNN.
        \item As $\NN$ is a $(t\facoA, \{\mathrm{id}\})$-GNN the simulation of its circuit families are also $t\facoA$-circuit families.
        \item The family of subcircuit described in Lemma \ref{lem:nn_nodes_zero} is also an $\faco$-circuit family. 
    \end{itemize}
    Therefore $\mathcal{K}$ is an $\facoA$-circuit family.
\end{proof}
\begin{figure}[h!]
        \centering
        \begin{tikzpicture}[
            every node/.style={minimum size =1cm, outer sep=0pt}, every path/.style={-stealth}
            ]
            \node[]     (M)                         {$\enc{M}$};
            \node[]     (v1)    [right=of M]        {$\ol{v}_1^{(1)}$};
            \node[]     (dots)  [right= of v1]      {$\dots$};
            \node[]     (vn)    [right= of dots]    {$\ol{v}_n^{(1)}$};
            
            \node[rectangle, draw, outer sep=0pt]       (n)     [fit=(M)(vn), below=4em of M.south west, anchor= west, inner sep=0] {determine neighborhood relations};
            
            \node[]     (sc)    [below =6em of n, fit=(n)]      {simulate the correct circuit from $\mathcal{C}^{(1)}$}; 
            \node[rectangle, draw, outer sep=0pt, minimum size =2cm]     (c)     [below =0.25em of sc]   {$C^{(1)}_k$};  
            \node[rectangle, draw, outer sep =0pt]      (box)   [fit={(sc) (c)}]    {};
            \node[]     (v1')   [below=of box]  {$\ol{v}_1^{(2)}$};

            \draw[]   (M.south)       --      (M.south |- n.north);
            \draw[]   (v1.south)      --      (v1.south |- n.north);
            \draw[]   (vn.south)      --      (vn.south |- n.north);
            \draw[]   ([xshift=-1cm]n.south)       --      ([xshift=-1cm]n.south  |- box.north) node[midway, left] {$\ol{v}_1^{(1)}$};
            \draw[]   ([xshift=0.5cm]n.south)       --     ([xshift=0.5cm]n.south  |- box.north) node[midway, left] {$\ol{u}_1^{(1)}$};
            \draw[]   ([xshift=2.5cm]n.south)       --       ([xshift=2.5cm]n.south  |- box.north) node[midway, right] {$\ol{u}_k^{(1)}$}; 
            \draw[draw=none]    ([xshift=1.5cm]n.south)       --      ([xshift=1.5cm]n.south  |- box.north) node[midway] {$\dots$};
            \draw[]   (box.south)     --  (v1');
        \end{tikzpicture}
        \caption{Proof of Theorem \ref{thm:cgnn_to_circ}, where $\ol{v}_i$ are the feature vectors of the C-GNN, $\ol{u}_j$ the feature vectors of the neighbors of a node. }
        \label{fig:cgnn_to_circ_idea}
    \end{figure}
    Other types of C-GNNs, mainly $(t\faco, \{\mathrm{id}\} \cup \mathcal{A})$-GNNs and $(t\faco, \{\mathrm{id}\})$-GNNs are also introduced in this paper. 
    Restricting Theorem \ref{thm:cgnn_to_circ} to $(t\facoA, \{\mathrm{id}\})$-GNNs is not a restraint on generality as the following corollary shows.
    \begin{corollary}
        Let $\mathcal{A}$ be a set of activation functions.
        Then $\facoA$-circuit families can simulate $(t\faco, \{\mathrm{id}\} \cup \mathcal{A})$-GNNs and $(t\faco, \{\mathrm{id}\})$-GNNs in the sense of Theorem~\ref{thm:cgnn_to_circ}.
    \end{corollary}
    \begin{proof}
        As $(t\facoA, \{\mathrm{id}\})$-GNNs have access to the functions in $\mathcal{A}$ at any point during the computations in a layer they are less restricted than $(t\faco, \{\mathrm{id}\} \cup \mathcal{A})$-GNNs and $(t\faco, \{\mathrm{id}\})$-GNNs, since those can either only apply the functions from $\mathcal{A}$ as activation functions to all nodes at the end of a layers computation or do not have access to them at all. 
        This is especially relevant for functions which cannot be computed by $t\faco$-circuit families.
    \end{proof}
    
\subsection{Proofs of Section \ref{sec:sim_circ_with_cgnn}}
\label{sec:app_3}
    The complete proofs of Section \ref{sec:sim_circ_with_cgnn} make use of the following lemma which allows us to 
    distinguish individual vectors from a finite subset of $\R^k$ while still following our desired complexity constraints.
    \begin{lemma}
        Let $A \subseteq \R^k$ be finite and let $\ol{a} \in A$.
        Then the function $\chi_{A, \ol{a}} \colon A \to \{\ol{0}, \ol{1}\}$ defined as
        \[
            \chi_{A, \ol{a}}(\ol{x}) \coloneqq 
            \begin{cases}
                \ol{1}, \textnormal{ if $\ol{x} = \ol{a}$}\\
                \ol{0}, \textnormal{ otherwise}
            \end{cases}
        \]
        where $\ol{1}$ is the $k$-dimensional vector of ones, respectively $\ol{0}$ the $k$-dimensional vector of zeros, is computable by a $\R^k$ circuit whose size only depends on $\lvert A \rvert$ and whose depth is constant.
    \label{lem:chiA}
    \end{lemma}
    \begin{proof}
        Let $A = \{\ol{a}, \ol{a}_1, \dots, \ol{a}_\ell\} \subseteq \R^k$ be a finite set and let the univariate polynomial $p_{A, \ol{a}} \colon A \to \R^k$ be defined as follows:
        \[
            p_{A, \ol{a}}(\ol{x}) \coloneqq \left( \prod_{1 \leq i \leq \ell} (\ol{a}_i - \ol{x}) \right) \cdot \frac{\ol{1}}{\prod_{1 \leq i \leq \ell} (\ol{a}_i - \ol{a})},
        \]
        where all operations are applied componentwise.
        Now, for every element $e \in A \setminus \{\ol{a}\}$, $p_{A, \ol{a}}(e)$ yields $\ol{0}$, since one factor in the left multiplication will be $\ol{0}$.
        On the other hand, $p_{A, \ol{a}}(\ol{a})$ evaluates to $\ol{1}$, since the left product yields exactly $\prod_{1 \leq i \leq \ell} (\ol{a_i} - \ol{a})$, which is what it gets divided by on the right. 
        Thus, $p_{A, \ol{a}} = \chi_{A, \ol{a}}$, and since for any given $A$ and $\ol{a} \in A$, $p_{A, \ol{a}}$ is a polynomial of constant degree, it can be evaluated by a circuit of constant depth and polynomial size in $\lvert A \rvert$.
    \end{proof}

\begin{definition}
\label{def:circ_to_cgnn_feature_graph_construction_full}
    Let $C$ be a circuit of an $\faco$-circuit family with $n$ inputs, $m$ outputs and numbered gates via an injective function $nr\colon V_{C} \rightarrow \mathbb{R}$ such that for all gates $g, g'$ of $C_n$: $\depth(g) < \depth(g') \implies nr(g) < nr(g')$. Let $C$ also be in path-length-normal form and let $\ol{x} \in \left(\R^k\right)^{n}$ be a tuple of vectors, the input of the circuit. 

    Then $\lgc{C}$ is the labeled graph $(G_C, g_{V_C})$, where $G_C=(V_C, E_C)$ is the underlying graph of the circuit and the function $g_{V_C}$ is defined as follows:
    The feature vector of a graph node corresponding to an input gate is the corresponding input value.
    Each other node gets its own number according to the gate numbering of the circuit as its feature vector.
\end{definition}
We now proceed to prove the theorems of Section \ref{sec:sim_circ_with_cgnn} in depth. 

As mentioned in the main paper our proofs of this section are formulated using $\R$ circuits. 
This technically results in our C-GNNs having feature vectors of dimension 1, i.\,e. vectors in $\R^1$. 
For simplicity reasons and because we also stated that $\R$ circuits and $\R^1$ circuits are equivalent we use real numbers, still called feature vectors, instead.
\begin{customthm}{\ref{thm:circ-gnn-without-fnc}}
    Let $\mathcal{C} = (C_n)_{n \in \N}$ be an $\mathrm{FAC}^0_{\R}$-circuit family.
    Then there exists a $(t\mathrm{FAC}^0_{\R}, \{\mathrm{id}\})$-GNN $\NN$, such that for all $n \in \N$ and $\ol{x} \in \R^n$,
    the feature vectors of the nodes of $f_{\NN}(\lgc{C_n})$ corresponding to output gates in $C_n$ are exactly 
    the components of $f_{C_n}(\ol{x})$. 
    The number of layers in $\NN$ is equal to the depth of $\mathcal{C}$.
\end{customthm}
\begin{proof}
    Without loss of generality let all circuits of $\mathcal{C}$ be in path-length normal form and let the gates in the circuits $C_n$ of $\mathcal{C}$ be numbered via an injective function $nr\colon V_{C_n} \rightarrow \mathbb{R}$ satisfying the conditions of Definition \ref{def:circ_to_cgnn_feature_graph_construction_full}.

    We then define $\NN$ as follows:
    Let $d$ be the depth of the circuits in $\mathcal{C}$.
    $\NN$ has $d$ layers and maps each number in $[d]$ to a circuit family $\C^{(i)}_\NN$, i.e. it assigns a circuit family to every layer $i$.
    Since each node $v_g$ in the given labeled graph $\lgc{C_n}$ 
    represents a gate $g$ in the underlying circuit $C_n$, each of the circuit families $\mathcal{C}^{(i)}_\NN$ of the C-GNN essentially make use of local neighborhood information in the graph of $C_n$. 
    The circuit family $\mathcal{C}^{(i)}_{\NN}$ uses the gate numbers in the gate numbering of $C_n$, the depth and gate type of the respective gates and the number of the unique successor of each gate.
    

    Let $v_g^{(i)}$ denote the feature vector of node $v_g$ in layer $i$. 
    As per the definition of $\lgc{C_n}$, the feature vector of each node in $\lgc{C_n}$ is initialized with the gate number of the gate in $C_n$ it represents, and the nodes representing input gates are initialized with the respective input values of $\ol{x}$.
    In each layer $i$, the respective circuit of $\CiN$ determines whether to perform a computation or just keep the feature vector as is. 
    The circuit gets the previous feature vector $v_g^{(i-1)}$ of the node as one of its inputs. 
    It then checks whether $v_g^{(i-1)}$ is among the gate numbers of gates with depth $i$ in $C_n$. 
    This is done in the following way: 
    Let $Y^{(i)}$ be the set of all gate numbers for gates in $C_n$ that have depth $i$. 
    Now $\CiN$ simply computes $\chi_{Y^{(i)}, y}\left(v_g^{(i-1)}\right)$ for each $y \in Y^{(i)}$ (see Lemma \ref{lem:chiA}) to check if $v_g^{(i-1)}$ is the number of a gate at depth $i$.
    If that is not the case, then the gate $g$ with $nr(g)=y$ is not at depth $i$ and should not be evaluated in layer $i$ and $\CiN$ keeps its feature vector as is, i.e., outputs the one it was given as its input. 
    Otherwise, $\CiN$ computes the respective function of $g$ in the following way:
    $\CiN$ computes $\chi_{X_{\textit{op}}, y}\left(v_g^{(i-1)}\right)$ for $\textit{op} \in \{+, \times, out\}$, where $X_{\textit{op}}$ contains the gate numbers of $\textit{op}$-gates in $C_n$, to determine which function it needs to evaluate.
    If $\textit{op} = +$, then $\CiN$ computes the sum of all the feature vectors in the neighborhood of $v_g$. 

    Since the neighborhood of each node in $\lgc{C_n}$ contains not only the predecessors of the node, but also the successor of the node, this value still needs to be removed from the aggregation. 
    This is achieved by figuring out which node is the successor of $v_g$ using the gate number of the gate corresponding to $v_g$. 
    We extract this information from the circuit $C_n$ when constructing the C-GNN circuits and use it in the manner of a lookup table.
    For every $y \in Y^{(i)}$ we check whether $v_g^{(i-1)}=y$, using the $\chi$ function from Lemma \ref{lem:chiA}, and if true subtract the gate number we saved for $y$, which is $nr(w)$ where $w$ is the successor of $v_g$. 
    \CiN proceeds analogously for nodes corresponding to $\times$ gates. For nodes corresponding to output gates \CiN computes the sum over the single neighbor. As all input circuits to the C-GNN are in path length normal form it is ensured that nodes corresponding to output gates are only evaluated in the last layer of the C-GNN.
    The complete functionality of $\CiN$ for each layer $i$ is described by Algorithm \ref{alg:add_mult_layer_circ}.
    By Lemma~\ref{lem:chiA} we know that for any finite set $A$ and any $a \in A$, $\chi_{A, a}$ can be evaluated with a circuit whose size only depends on $\lvert A \rvert$ and whose depth is constant.
    Since we only need to evaluate polynomially many such functions and we can compute them in parallel, the circuit family we outline here is still an $\facoR$-circuit family.


    
    For example a node corresponding to a $+$ gate with number $5$ at depth 2, which knows that the gate its successor  has number $9$, will sum up all of its neighbors values and subtract $9$ in the second layer of the C-GNN. 
    See Figure \ref{fig:example_comp_circ_cgnn_proof} for visualisation.
    The only activation function used in each C-GNN layer is the identity function.
    \end{proof}
    \begin{algorithm}[h!]
    \caption{Proof Theorem \ref{thm:circ-gnn-without-fnc}: algorithm for the circuit family in layer $i$ of the C-GNN}
    \label{alg:add_mult_layer_circ}
    \begin{algorithmic}
        \STATE {\bfseries Input:} $v_g^{(i-1)}$, $U_g = \{u^{(i-1)} \mid u \in \mathcal{N}_G(v_g)\}$\
        \STATE $Y^{(i)} = \{ nr(g) \mid \depth(g) =i\}$\; \COMMENT{The set of numbers of gates that need to be evaluated in this layer} 
        \STATE $X_+ = \{ nr(g) \mid g \text{ is + gate} \}$\;
        \STATE $X_{\times} = \{ nr(g) \mid g \text{ is } \times \text{ gate} \}$\;
        \STATE $X_{out} = \{ nr(g) \mid g \text{ is } out \text{ gate} \}$\;
            \IF[If a node needs to be evaluated in this layer] {$v_g^{(i-1)} \in Y^{(i)}$} 
                \IF[If $v_g$ corresponds to an addition gate] {$v_g^{(i-1)}\in X_{+}$}
                \STATE $v_g^{(i)} = \sum U_g - nr\left(suc\left(v_g\right)\right)$\; 
                \ELSIF[If $v_g$ corresponds to a multiplication gate]{$v_g^{(i-1)}\in X_{\times}$}
                 \STATE $v_g^{(i)} = \prod U_g / nr\left(suc\left(v_g\right)\right)$\; 
                \ELSE[If $v_g$ corresponds to an output gate]
                \STATE $v_g^{(i)} = \sum U_g$\;
                \ENDIF
            \ELSE[If $v_g$ does not correspond to a gate at depth $i$]
            \STATE $v_g^{(i)} = v_g^{(i-1)}$\;
            \ENDIF
    \end{algorithmic}
    \end{algorithm}
    
    \begin{figure}[h!]
    \centering
    \begin{subfigure}[htb]{\textwidth}
        \centering
        \resizebox{.8\linewidth}{!}{
            \begin{subfigure}[b]{0.5\textwidth}
                \centering
                \begin{tikzpicture}[gate/.style={circle, draw, minimum size =1cm}, every path/.style={-stealth}]
                    \node[gate](v5){$+$};
                    \node[] (nr5) [right=0.5cm of v5] {$nr\left(g_+\right)=5$};
                    \node[gate](v9)[below=of v5]{$z$};
                    \node[] (nr9)[right=0.5cm of v9]{$nr\left(g_z\right)=9$};
                    \node[gate] (vx) [above left=of v5] {$x$};
                    \node[gate] (vy) [above right=of v5] {$y$};
                    \node[] (depth2) [left=2cm of v5] {\textbf{depth 2}};
                    \node[] (depth3) [left=2cm of v9] {\textbf{depth 3}};
                    \node[] (depth1) [above= of depth2] {\textbf{depth 1}};
                    \node[] (dots) [below = of v9] {$\vdots$};
                    \node[] (inx) [above= of vx] {$3$};
                    \node[] (iny) [above =of vy] {$7$};
                    \draw[] (vx) -- (v5);
                    \draw[] (vy) -- (v5);
                    \draw[] (v5) -- (v9);
                    \draw[] (v9) -- (dots);
                    \draw[] (inx) -- (vx);
                    \draw[] (iny) -- (vy);
                \end{tikzpicture}
                \caption{Initial circuit}
            \end{subfigure}
            \begin{subfigure}[b]{0.5\textwidth}
                \centering
                \begin{tikzpicture}[gate/.style={circle, draw, minimum size =1cm}]
                    \node[gate](v5){$v_{g_+}$};
                    \node[] (v5_val) [right =0.1cm of v5] {$v_{g_+}^{(1)}=5$};
                    \node[gate](v9)[below=of v5]{$v_{g_z}$};
                    \node[] (v9_val) [right =0.1cm of v9] {$v_{g_z}^{(1)}=9$};   
                    \node[gate] (vx) [above left=of v5] {$v_{g_x}$};
                    \node[] (vx_val) [right =0.1cm of vx] {$v_{g_x}^{(1)}=3$};
                    \node[gate] (vy) [above right=of v5] {$v_{g_y}$};
                    \node[] (vy_val) [right =0.1cm of vy] {$v_{g_y}^{(1)}=7$};
                    \node[] (dots) [below = of v9] {$\vdots$};
                    \node[] (xdots) [above =of vx] {$\vdots$};
                    \node[] (ydots) [above =of vy] {$\vdots$};
                    \draw[-] (vx) -- (v5);
                    \draw[-] (vy) -- (v5);
                    \draw[-] (v5) -- (v9);
                    \draw[-] (v9) -- (dots);
                    \draw[-] (xdots) -- (vx);
                    \draw[-] (ydots) -- (vy);
                \end{tikzpicture}
                \caption{Feature graph with feature vectors}
            \end{subfigure}
        }
        \caption{Before the computation  of the layer of the C-GNN}
    \end{subfigure}
    \begin{subfigure}[htb]{\textwidth}
        \centering
        \resizebox{.8\linewidth}{!}{
            \begin{subfigure}[b]{0.5\textwidth}
                \centering
                \begin{tikzpicture}[
                    gate/.style={circle, draw, minimum size =1cm}, every path/.style={-stealth}
                    ]
                    \node[gate](v5){$+$};
                    \node[] (nr5) [right=0.5 cm of v5] {$nr\left(g_+\right)=5$};
                    \node[gate](v9)[below=of v5]{$z$};
                    \node[] (nr9)[right=0.5cm of v9]{$nr\left(g_z\right)=9$};
                    \node[gate] (vx) [above left=of v5] {$x$};
                    \node[gate] (vy) [above right=of v5] {$y$};
                    \node[] (depth2) [left=2cm of v5] {\textbf{depth 2}};
                    \node[] (depth3) [left=2cm of v9] {\textbf{depth 3}};
                    \node[] (depth1) [above= of depth2] {\textbf{depth 1}};
                    \node[] (dots) [below = of v9] {$\vdots$};
                    \node[] (inx) [above= of vx] {$3$};
                    \node[] (iny) [above =of vy] {$7$};
                    \draw[] (vx) -- (v5);
                    \draw[] (vy) -- (v5);
                    \draw[] (v5) -- (v9);
                    \draw[] (v9) -- (dots);
                    \draw[] (inx) -- (vx);
                    \draw[] (iny) -- (vy);
                \end{tikzpicture}
                \caption{Initial circuit}
            \end{subfigure}
            \begin{subfigure}[b]{0.5\textwidth}
                \centering
                \begin{tikzpicture}[
                    gate/.style={circle, draw, minimum size =1cm}
                    ]
                    \node[gate](v5){$v_{g_+}$};
                    \node[] (v5_val) [right =0.1cm of v5] {$v_{g_+}^{(2)}=10$};
                    \node[gate](v9)[below=of v5]{$v_{g_z}$};
                    \node[] (v9_val) [right =0.1cm of v9] {$v_{g_z}^{(2)}=9$};   
                    \node[gate] (vx) [above left=of v5] {$v_{g_x}$};
                    \node[gate] (vy) [above right=of v5] {$v_{g_y}$};
                    \node[] (dots) [below = of v9] {$\vdots$};
                    \node[] (xdots) [above =of vx] {$\vdots$};
                    \node[] (ydots) [above =of vy] {$\vdots$};
                    \draw[-] (vx) -- (v5);
                    \draw[-] (vy) -- (v5);
                    \draw[-] (v5) -- (v9);
                    \draw[-] (v9) -- (dots);
                    \draw[-] (xdots) -- (vx);
                    \draw[-] (ydots) -- (vy);
                \end{tikzpicture}
                \caption{Feature graph with feature vectors}
            \end{subfigure}
        }
        \caption{After the computation of the layer of the C-GNN}
    \end{subfigure}
    \caption{Example in the proof of Theorem \ref{thm:circ-gnn-without-fnc}\protect \footnotemark } 
    \label{fig:example_comp_circ_cgnn_proof}
\end{figure}
\footnotetext{For readability purposes the one-dimensional vectors in the C-GNN are written as real numbers.}
\begin{customthm}{\ref{thm:circ-cgnn-fnc-in-circ}}
    Let $\mathcal{C} = (C_n)_{n \in \N}$ be an $\mathrm{FAC}^0_{\R}[\mathcal{A}]$-circuit family.
    Then there exists a $(t\mathrm{FAC}^0_{\R}[\mathcal{A}], \{\mathrm{id}\})$-GNN $\NN$, such that for all $n \in \N$ and $\ol{x} \in \R^n$,
    the feature vectors of the nodes of $f_{\NN}\left(\lgc{C_n}\right)$ corresponding to output gates in $C_n$ are exactly 
    the components of $f_{C_n}(\ol{x})$.
    The number of layers in $\NN$ is equal to the depth of $\mathcal{C}$.
\end{customthm}
    
\begin{proof}
    This proof follows the same ideas as the proof of Theorem \ref{thm:circ-gnn-without-fnc} but includes the occurrence of gates computing functions besides + and $\times$ in the circuit family $\mathcal{C}$ as well as in the circuit families \CiN used in the C-GNN.

    Without loss of generality let all circuits of $\mathcal{C}$ be in path-length normal form and let the gates in the circuits $C_n$ of $\mathcal{C}$ be numbered via an injective function as described in Definition \ref{def:circ_to_cgnn_feature_graph_construction_full}.
    We will use $nr(C_n)$ to denote the set of all gate numbers in the circuit $C_n$.
    
    We now define \NN as follows:
    For every circuit in all circuit families \CiN, we initialize all nodes with their respective gate number or input value as described in the proof of Theorem \ref{thm:circ-gnn-without-fnc}.
    The sets $Y^{(i)}$ for all layers $i$ of the \C-GNN and the sets $X_+$, $X_{\times}$ and $X_\textit{out}$ are defined the same as well. 
    Additionally, we now need sets $X_\sigma$ for every function $\sigma \in \mathcal{A}$, such that $X_\sigma$ contains all gate numbers of gates computing $\sigma$ in $C_n$. 
    All these sets are finite, since they are all subsets of $nr(C_n)$ which is the set of gate numbers of $C_n$.
    
    The idea is to handle the nodes corresponding to $\sigma$-gates in the same way as the other nodes -- for every node $v_g$ we check if it is to be evaluated in this layer, i.e. the depth of gate $g$ in $C_n$ equals the current layer $i$, by checking whether $v_g^{(i-1)} \in Y^{(i)}$ and if so, assess of which type $g$ is by checking $v_g^{(i-1)}$ for inclusion in a set $X_{\textit{op}}$ where $\textit{op} \in \{ +, \times, \textit{out} \} \cup \mathcal{A}$.
    We do this as in the proof of Theorem \ref{thm:circ-gnn-without-fnc} by computing $\chi_{X_{\textit{op}}, x}\left(v_g^{(i-1)}\right)$ for all $X_{\textit{op}}$ and all $x \in X_{\textit{op}}$. 
    
    If we determined that a gate $g$ corresponding to the node $v_g$ computes a function $\sigma \in \mathcal{A}$, we compute the function on the feature vector of the node corresponding to its predecessor in $C_n$. 
    Since all $\sigma \in \mathcal{A}$ are unary functions, $g$ has only one predecessor and given that $C_n$ is in path-length normal form, it also has only one successor. 
    Let $u$ be the predecessor and $w$ be the successor of $v_g$. 
    We compute $v_g^{(i)}=\sigma\left(u^{(i-1)}\right)$ by summing up the feature vectors of all neighbors of $v_g$ in $\lgc{C_n}$ and then determining the gate number of the gate corresponding to $w$ and subtract this value from the sum as described in the proof of Theorem \ref{thm:circ-gnn-without-fnc}.
    This leaves us with the feature vector of the predecessor of $v_g$, since $u^{(i-1)}+w^{(i-1)}-w^{(i-1)}=u^{(i-1)}$.
    Then $\sigma$ is applied to $u^{(i-1)}$ using an $\sigma$-gate and $v_g^{(i)}$ is set to this value. 

    For all nodes $v_h$ which do not represent gates which have depth $i$ in $C_n$, the new feature vector $v_h^{(i)}$ equals $v_h^{(i-1)}.$ 
    
    All layers of \NN use the identity as their activation function.
\end{proof}

\begin{customthm}{\ref{thm:circ-gnn-fnc}}
    Let $\mathcal{A}$ be a set of countably injective activation functions and let $\mathcal{C} = (C_n)_{n \in \N}$ be an $f\mathrm{FAC}^0_{\R}[\mathcal{A}]$-circuit family. 
    Then there exists a $(t\mathrm{FAC}^0_{\R}, \mathcal{A} \cup \{ \mathrm{id}\})$-GNN $\NN$, such that for all $n \in \N$ and $\ol{x} \in (\R)^n$,
    the feature vectors of the nodes of $f_{\NN}\left(\lgc{C_n}\right)$ corresponding to output gates in $C_n$ are exactly 
    the components of $f_{C_n}(\ol{x})$.
    The number of layers in $\NN$ is equal to the depth of $\mathcal{C}$.
\end{customthm}

\begin{proof}
We again reuse the basic concept of the proof of Theorem~\ref{thm:circ-gnn-without-fnc} but adapt it for function gates. 
Unlike in Theorem~\ref{thm:circ-cgnn-fnc-in-circ}, the circuit families of our $(t\mathrm{FAC}^0_{\R}, \mathcal{A} \cup \{ \mathrm{id}\})$-GNN do not have arbitrary access to functions in $\mathcal{A}$ and thus cannot decide for each individual node, whether to apply an activation or not. 
We can only apply it to all of our nodes' feature vectors, or to none of them. 

Let $\sigma \in \mathcal{A}$ be countably injective on a set $S \subseteq \R$ and let the gates in $C_n$ be numbered as described in Definition \ref{def:circ_to_cgnn_feature_graph_construction_full} by values from $S$.
Let $i \in \N$ be such that all gates at depth $i$ in $C_n$ are $\sigma$-gates.
Since $C_n$ is in function-layer form, if any gate at depth $i$ is a $\sigma$-gate, then all gates at depth $i$ are.
The circuit family $\mathcal{C}_\NN^{(i)}$ used for each node in layer $i$ in $\NN$ then first checks whether its previous feature vector $v_g^{(i-1)}$ corresponds to a gate number of a gate $g$ in layer $i$ by querying $v_g^{(i-1)} \in Y^{(i)}$ using $\chi_{Y^{(i)},y}\left(v_g^{(i-1)}\right)$, where $Y^{(i)}$ is the set of all gate numbers for gates in $C_n$ that have depth $i$, as introduced before.

If $v_g^{(i-1)} \in Y^{(i)}$, then $\mathcal{C}_\NN^{(i)}$ proceeds as in the case for addition gates in Algorithm~\ref{alg:add_mult_layer_circ}, summing over its neighbors and removing the value of its successor.
This way, after the activation function is applied, $v_g^{(i)}=\sigma\left(v_g^{(i-1)}\right)$, the intended value after layer $i$ in $\NN$.
\begin{algorithm}[htb]
        \caption{Proof Theorem \ref{thm:circ-gnn-fnc}: algorithm for the circuit in layer $i$ of the C-GNN if layer $i$ evaluates function gates}
        \label{alg:func_layer_circ}
        \begin{algorithmic}
        \STATE {\bfseries Input:} $v_g^{(i-1)}$, $U_g = \{u^{(i-1)} \mid u \in \mathcal{N}_G(v_g)\}$\;
        \STATE $Y^{(i)} = \{ nr(g) \mid \depth(g) =i\}$\; \COMMENT{The set of numbers of gates that need to be evaluated in this layer}
        \STATE $X_+ = \{ nr(g) \mid g \text{ is $+$ gate} \}$\;
        \STATE $X_{\times} = \{ nr(g) \mid g \text{ is } \times \text{ gate} \}$\;
        \STATE $X_{out} = \{ nr(g) \mid g \text{ is } \textit{out} \text{ gate} \}$\;
        \IF[If a node needs to be evaluated in this layer]{$v_g^{(i-1)} \in Y^{(i)}$}  
            \IF[If $v_g$ corresponds to addition gate]{$v_g^{(i-1)}\in X_{+}$}
            \STATE $v_g^{(i)} = \sum U_g - nr\left(suc\left(v_g\right)\right)$\; 
            \ELSIF[If $v_g$ corresponds to multiplication gate]{$v_g^{(i-1)}\in \sigma\left(X_{\times}\right)$}
            \STATE $v_g^{(i)} = \prod U_g / nr\left(suc\left(v_g\right)\right)$\; 
            \ELSE[If $v_g$ corresponds to output gate]
            \STATE $v_g^{(i)} = \sum U_g$
            \ENDIF
        \ELSE[If $v_g$ does not correspond to a gate at depth $i$]
        \STATE $v_g^{(i)} = \sigma^{-1}\left(v_g^{(i-1)}\right)$\; 
        \ENDIF
        \end{algorithmic}
    \end{algorithm}
If $v_g^{(i-1)} \notin Y^{(i)}$, we need to make sure that after the application of $\sigma$ as an activation function, $v_g^{(i)}$ is still the gate number of its corresponding gate $g$ in $C_n$.
For this purpose, $\mathcal{C}_\NN^{(i)}$ determines and outputs $\sigma^{-1}\left(v_g^{(i-1)}\right)$, so that after applying the activation function, $v_g^{(i)} = \sigma\left(\sigma^{-1}\left(v_g^{(i-1)}\right)\right) = v_g^{(i-1)}$.
The inverse $\sigma^{-1}\left(v_g^{(i-1)}\right)$ exists, because $\sigma$ is countably injective in the set of gate numbers of $C_n$.
The circuit family $\mathcal{C}_\NN^{(i)}$ sets its output to $\sigma^{-1}\left(v_g^{(i-1)}\right)$ by computing $\sum_{v \in \textit{nr}(C_n)} \sigma^{-1}(v) \cdot \chi_{S,v}\left(v_g^{(i-1)}\right)$, where $\textit{nr}(C_n)$ is the set of gate numbers of $C_n$.
This works for all gates at depths $>i$.
Gates at depths $<i$ do not have an influence on the output of the computation anymore, and what happens to them is thus of no importance.

The procedure performed by $\mathcal{C}_\NN^{(i)}$ is outlined in Algorithm~\ref{alg:func_layer_circ}.

The preceding procedure requires that all functions in $\mathcal{A}$ are countably injective in the same set.
If they are not countably injective in the same set, we modify the construction by using a new gate numbering of $C_n$ for each function layer, assigning new values to the feature vectors before the activation function is used. 
All sets of numbers of specific gates, e.g. numbers of addition gates are updated accordingly.

\end{proof}

\end{document}